\documentclass{article}

\usepackage[utf8]{inputenc} 
\usepackage[T1]{fontenc}    
\usepackage[round]{natbib}
\usepackage[english]{babel}
\usepackage{hyperref}
\usepackage{amsmath,amsthm,amssymb}
\usepackage[left=3cm,right=3cm]{geometry}

\author{anonymous}
\usepackage{xcolor}
\definecolor{linkcolor}{RGB}{83,83,182}
\definecolor{citecolor}{RGB}{128,0,128}
\hypersetup{
    colorlinks=true,
    citecolor=citecolor,
    linkcolor=linkcolor,
    urlcolor=linkcolor
}

\newcommand{\sqrtlasso}{$\sqrt{\mathrm{Lasso}}$\xspace}
\usepackage[capitalise,nameinlink]{cleveref} 
\usepackage{booktabs}  
\usepackage{quentin_shortcuts}

\usepackage[accepted]{aistats2020}
\usepackage{chngcntr}  
\usepackage{apptools}
\usepackage{siunitx}
\AtAppendix{\counterwithin{lemma}{section}}
\AtAppendix{\counterwithin{lemmaenumi}{section}}
\AtAppendix{\counterwithin{theorem}{section}}
\newcommand\numberthis{\addtocounter{equation}{1}\tag{\theequation}}  

\begin{document}
\twocolumn[
\aistatstitle{Support recovery and sup-norm convergence rates for sparse pivotal estimation}
\aistatsauthor{ Mathurin Massias${}^*$ \And Quentin Bertrand${}^*$ \And Alexandre Gramfort \And Joseph Salmon}
\aistatsaddress{
Universit\'e Paris-Saclay \\ Inria, CEA  \\ Palaiseau, France \And
Universit\'e Paris-Saclay \\ Inria, CEA  \\ Palaiseau, France \And
Universit\'e Paris-Saclay \\ Inria, CEA  \\ Palaiseau, France \And
IMAG \\ Univ. Montpellier, CNRS \\ Montpellier, France}
 ]

\begin{abstract}


In high dimensional sparse regression, pivotal estimators are estimators for which the optimal regularization parameter is independent of the noise level.
The canonical pivotal estimator is the square-root Lasso, formulated along with its derivatives as a ``non-smooth + non-smooth'' optimization problem.
Modern techniques to solve these include smoothing the datafitting term, to benefit from fast efficient proximal algorithms.
In this work we show minimax sup-norm convergence rates for non smoothed and smoothed, single task and multitask square-root Lasso-type estimators.
Thanks to our theoretical analysis, we provide some guidelines on how to set the smoothing hyperparameter, and illustrate on synthetic data the interest of such guidelines.

\end{abstract}


\section{Introduction}
\label{sec:introduction}
Since the mid 1990's and the development on the Lasso \citep{Tibshirani96}, a vast literature has been devoted to sparse regularization for high dimensional regression.
Statistical analysis of the Lasso showed that it achieves optimal rates (up to log factor, \citealt{Bickel_Ritov_Tsybakov09}); see also \cite{Buhlmann_vandeGeer11} for an extensive review.
Yet, this estimator requires a specific calibration to achieve such an appealing rate: the regularization parameter must be proportional to the noise level. 
This quantity is generally unknown to the practitioner, hence the development of methods which are adaptive \wrt the noise level.
An interesting candidate with such a property is the square-root Lasso (\sqrtlasso, \citealt{Belloni_Chernozhukov_Wang11}) defined for an observation vector $y \in \bbR^n$, a design matrix $X \in \bbR^{n\times p}$ and a regularization parameter $\lambda$ by
\begin{problem}\label{pb:sqrt}
  \argmin_{\beta \in \bbR^p}
  \frac{1}{\sqrt{n}} \normin{y - X\beta}_2 + \lambda \normin{\beta}_1 \enspace.
\end{problem}
It has been shown to be \emph{pivotal} with respect to the noise level by \citet{Belloni_Chernozhukov_Wang11}: the optimal regularization parameter of their analysis does not depend on the true noise level.
This feature is also encountered in practice as illustrated by \Cref{fig:lambda_opt_noise} (see details on the framework in \Cref{sub:lambda_opt_noise}).

\begin{figure}
  \centering
  \includegraphics[width=0.48\linewidth]{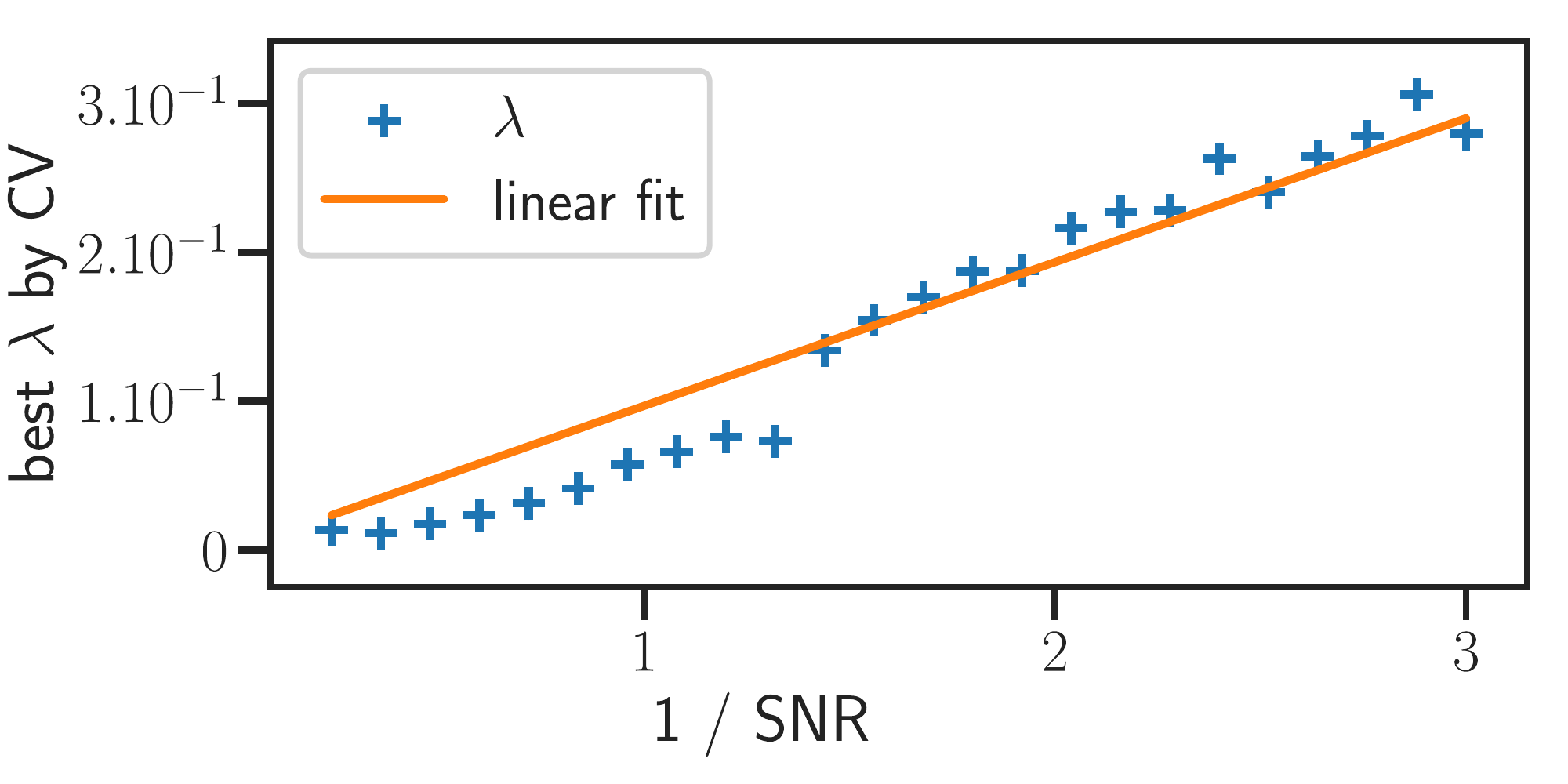}
  \includegraphics[width=0.48\linewidth]{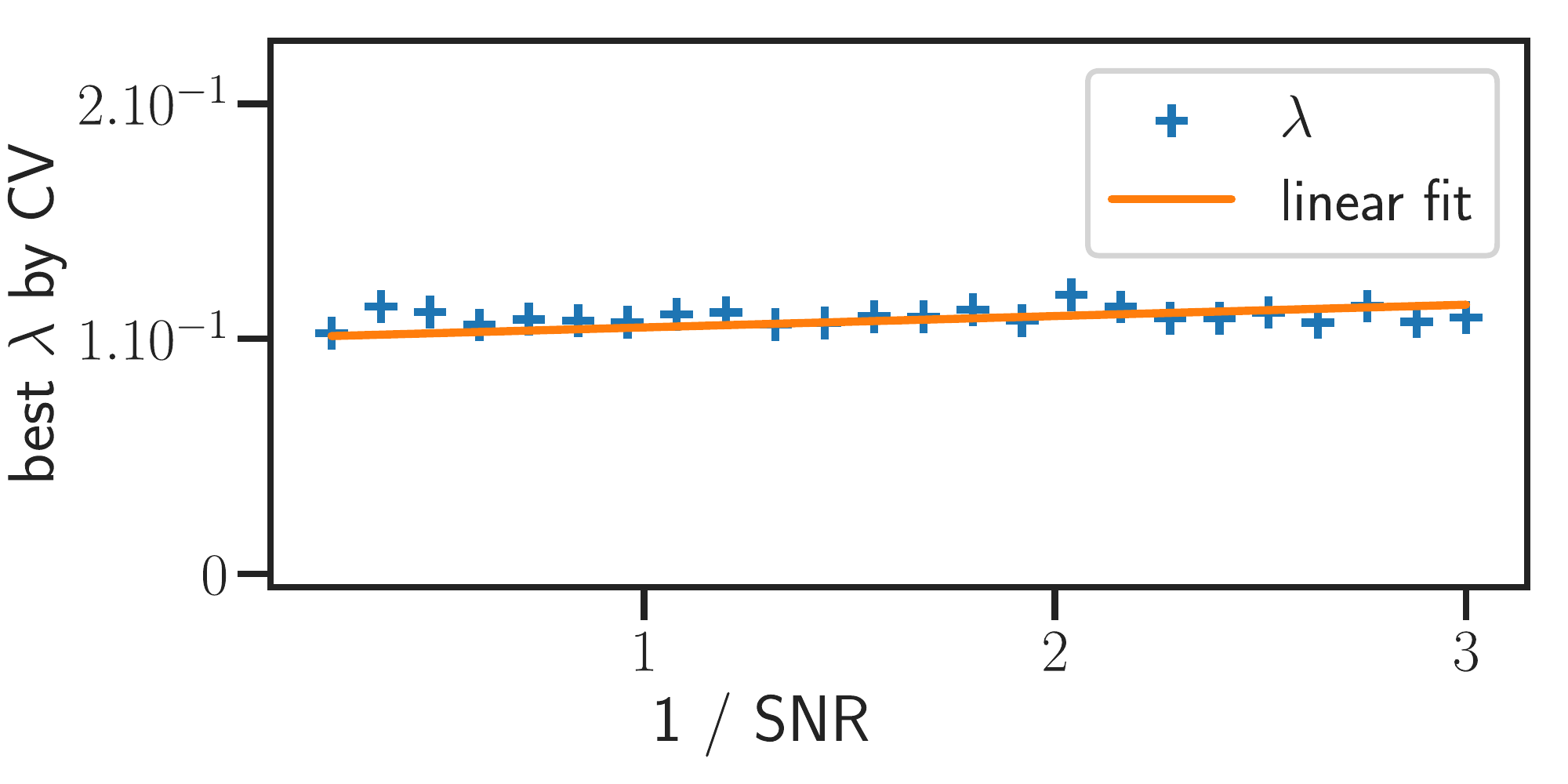}
  \caption{ Lasso (left) and square-root Lasso (right) optimal regularization parameters $\lambda$ determined by cross validation on prediction error (blue), as a function of the noise level on simulated values of $y$.
  As indicated by theory, the Lasso's optimal $\lambda$ grows linearly with the noise level, while it remains constant for the square-root Lasso.}
  \label{fig:lambda_opt_noise}
\end{figure}

Despite this theoretical benefit, solving the square-root Lasso requires tackling a ``non-smooth + non-smooth'' optimization problem.
To do so, one can resort to conic programming \citep{Belloni_Chernozhukov_Wang11} or primal-dual algorithms \citep{Chambolle_Pock11} for which practical convergence may rely on hard-to-tune hyper-parameters.
Another approach is to use variational formulations of norms, \eg expressing the absolute value as $\absin{x} = \min_{\sigma > 0} \frac{x^2}{2 \sigma} + \frac{\sigma}{2}$ (\citealt[Sec. 5.1]{Bach_Jenatton_Mairal_Obozinski12}, \citealt{Micchelli_Morales_Pontil10}).
This leads to \emph{concomitant estimation} \citep{Huber_Dutter74}, that is, optimization problems over the regression parameters and an additional variable.
In sparse regression, the seminal concomitant approach is the concomitant~Lasso~\citep{Owen07}:
\begin{problem}\label{pb:conco}
  \argmin_{\substack{\beta \in \bbR^{p }, \sigma > 0 }}
  \frac{1}{2n\sigma} \normin{y - X \beta}_{2}^2
  + \frac{\sigma}{2}
  + \lambda \norm{\beta}_{1}
  \enspace ,
\end{problem}
which yields the same estimate $\hat \beta$ as \Cref{pb:sqrt} whenever $y - X \hat \beta \neq 0$.
\Cref{pb:conco} is more amenable: it is jointly convex, and the datafitting term is differentiable.
Nevertheless, the datafitting term is still not smooth, as $\sigma$ can approach 0 arbitrarily: proximal solvers cannot be applied safely.
A solution is to introduce a constraint $\sigma \geq \sigmamin$ \citep{Ndiaye_Fercoq_Gramfort_Leclere_Salmon16}, which amounts to \emph{smoothing} \citep{Nesterov05,Beck_Teboulle12} the square-root Lasso,
\ie replacing its non-smooth datafit by a smooth approximation (see details in \Cref{sub:smoothing}).

There exist a straightforward way to generalize the square-root Lasso to the multitask setting
(observations $Y \in \bbR^{n\times q}$): the multitask square-root Lasso,
\begin{problem}\label{pb:multitask_sqrt}
    \argmin_{\Beta \in \bbR^{p\times q}}
    \frac{1}{\sqrt{nq}} \norm{Y - X \Beta}_F
    + \lambda \norm{\Beta}_{2,1}
    \enspace ,
\end{problem}
where $\norm{\Beta}_{2, 1}$ is the $\ell_1$ norm of the $\ell_2$ norms of the rows.
Another extension of the square-root Lasso to the multitask case  is the multivariate square-root Lasso\footnote{modified here with a row-sparse penalty instead of $\ell_1$}
\citep[Sec. 3.8]{vandeGeer16}:
\begin{problem}\label{pb:multivar_sqrt}
    \argmin_{\Beta \in \bbR^{p\times q}}
    \frac{1}{\sqrt{nq (n \wedge q)}} \norm{Y - X \Beta}_*
    + \lambda \norm{\Beta}_{2,1}
    \enspace .
\end{problem}
It is also shown by \citet{vandeGeer16} that when $Y - X \hat \Beta$ is full rank, \Cref{pb:multivar_sqrt} also admits a concomitant formulation, this time with an additional matrix variable:
\begin{problem}\label{pb:conco_multivar}
  \argmin_{\substack{\Beta \in \bbR^{p \times q}\\ \Snoise \succ 0 }}
  \frac{1}{2nq} \normin{Y - X \Beta}_{\Snoise^{-1}}^2
  + \frac{1}{2n} \Tr(\Snoise)
  + \lambda \norm{\Beta}_{2, 1}
  \enspace .
\end{problem}
In the analysis of the square-root Lasso \eqref{pb:sqrt}, the non-differentiability at $0$ can be avoided by excluding the corner case where the residuals $y - X \hat\beta$ vanish.
However, analysis of the multivariate square-root Lasso through its concomitant formulation \eqref{pb:conco_multivar} has a clear weakness: it requires excluding rank deficient residuals cases, which is far from being a corner case.
As illustrated in \Cref{fig:sing_val_residuals}, the full rank assumption made by \citet[Lemma 1]{vandeGeer_Stucky16} or \citet[Rem. 1]{Molstad19} is not realistic, even for $q \geq n$ and high values of $\lambda$ (see \Cref{sec:experiments} for the setting's details).
Motivated by numerical applications, \citet{Massias_Fercoq_Gramfort_Salmon17} introduced a lower bound on the smallest eigenvalue of $S$ ($S \succeq \sigmamin \Id_n$) in \Cref{pb:conco_multivar} to circumvent this issue.
As observed by \citet[Sec 3.1]{Bertrand_Massias_Gramfort_Salmon19a}, this amounts to smoothing the nuclear norm.

Our goal is to prove sup-norm convergence rates and support recovery guarantees for the estimators introduced above, and their smoothed counterparts.
\begin{figure}
  \centering
  \includegraphics[width=0.9\linewidth]{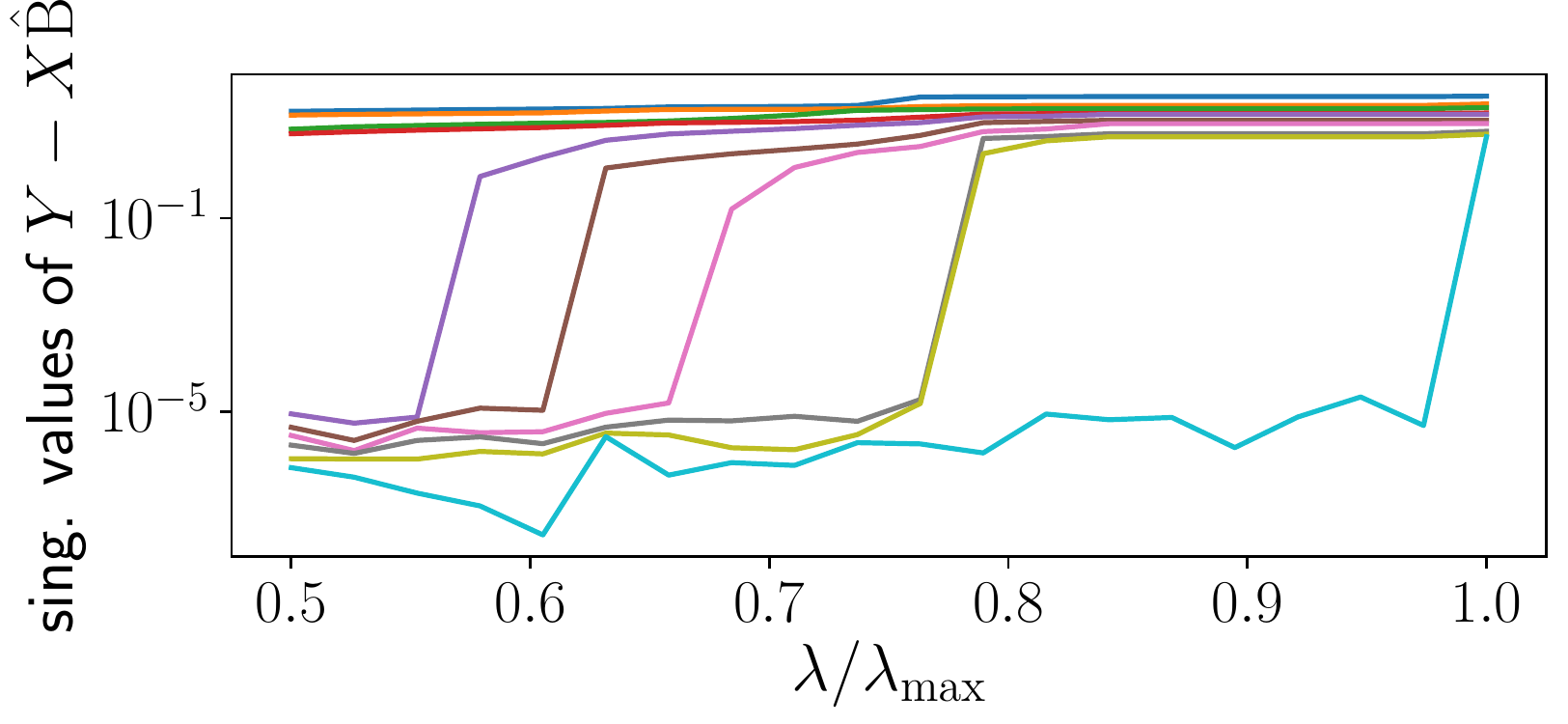}
  \caption{Singular values of the residuals $Y - X \hat \Beta$ of the multivariate square-root Lasso ($n=10, q=20, p =30$), as a function of $\lambda$. The observation matrix $Y$ is full rank, but the residuals are rank deficient even for high values of the regularization parameter, invalidating the classical assumptions needed for statistical analysis.
  }
  \label{fig:sing_val_residuals}
\end{figure}
%
%
\paragraph{Related works}
The statistical properties of the Lasso have been studied under various frameworks and assumptions.
\citet{Bickel_Ritov_Tsybakov09} showed that with high probability, $\normin{X(\hat \beta - \beta^*)}_2$ vanishes at the minimax rate (prediction convergence), whereas \citet{Lounici08} proved the sup-norm convergence and the support recovery of the Lasso (estimation convergence), \ie controlled the quantity $\normin{\hat \beta - \beta^*}_\infty$.
The latter result was extended to the multitask case by \citet{Lounici_Pontil_vandeGeer_Tsybakov11}.

Since then, other Lasso-type estimators have been proposed and studied, such as the square-root Lasso \citep{Belloni_Chernozhukov_Wang11} or the scaled Lasso \citep{Sun_Zhang12}.
In the multitask case, \citet{Liu_Wang_Zhao15} introduced the Calibrated Multivariate Regression, and \citet{vandeGeer_Stucky16,Molstad19} studied the multivariate square-root Lasso.
These estimators have been proved to converge \emph{in prediction}. 
However, apart from \citet{Bunea_Lederer_She13} for a particular group square-root Lasso, we are not aware of other works showing sup-norm convergence\footnote{of particular interest: combined with a large coefficients assumption, it implies support identification} of these estimators.
%

%
Within the framework introduced by \citet{Lounici08},
our contributions are the following:
\begin{itemize}[itemsep=0pt,topsep=0pt,partopsep=0pt]
  \item We prove sup-norm convergence and support recovery of the multitask square-root Lasso and its smoothed version.
  \item We prove sup-norm convergence and support recovery
    of the \emph{multivariate square-root Lasso} \citep[Sec. 2.2]{vandeGeer_Stucky16},
    and a smoothed version of it.
  \item Theoretical analysis leads to guidelines for the setting of the smoothing parameter $\sigmamin$.
  In particular, as soon as $\sigmamin \leq \sigma^* / \sqrt{2}$, the ``optimal'' $\lambda$ and the sup-norm bounds obtained do not depend on $\sigmamin$.
  \item We show on synthetic data the support recovery performances are little sensitive to the smoothing parameter $\sigmamin$ as long as $\sigmamin \leq \sigma^* / \sqrt{2}$.
\end{itemize}
Our contributions with respect to the existing literature are summarized in \Cref{tab:summary_estimators}.

\paragraph{Notation}
Columns and rows of matrices are denoted by $A_{:i}$ and $A_{i:}$ respectively.
For any $\Beta \in \bbR^{p \times q}$ we define $\cS(\Beta) \eqdef \condsetin{j \in [p]}{||\Beta_{j:}||_2 \neq 0}$
the row-wise support of $\Beta$.
We write $\cS_*$ for the row-wise support of the true coefficient matrix $\Beta^* \in \bbR^{p \times q}$.
For any $\Beta \in \bbR^{p \times q}$ and any subset $\cS$ of $[p]$ we denote $\Beta_\cS$ the matrix in $\bbR^{p \times q}$ which has the same values as $\Beta$ on the rows with indices in $\cS$ and vanishes on the complement $\cS^c$.
The estimated regression coefficients are written $\hat \Beta$, their difference with the true parameter $\Beta^*$ is noted $\Delta \eqdef \hat \Beta - \Beta^*$.
The residuals at the optimum are noted $\Hat \Epsilon \eqdef Y - X\hat\Beta$.
The infimal convolution between two functions $f_1$ and $f_2$ from $\bbR^d$ to $\bbR$ is denoted by $f_1\infconv f_2$ and is defined for any $x$ as $\inf\{f_1(x-y)+f_2(y) : y\in \bbR^d\}$.
For $a < b$, $[x]^b_a \eqdef \max(a, \min(x, b))$ is the clipping of $x$ at levels $a$ and $b$.
The Frobenius and nuclear norms are denoted by $\normin{\cdot}_F$ and $\normin{\cdot}_*$ respectively.
For matrices, $\normin{\cdot}_{2,1}$ and $\normin{\cdot}_{2, \infty}$ are the row wise $\ell_{2, 1}$ and $\ell_{2, \infty}$ norms, \ie respectively the sum and maximum of rows norms.
The subdifferential of a function $f$ is denoted $\partial f$, and its Fenchel conjugate is written $f^*$, equal at $u$ to $\sup_x \langle u, x \rangle - f(x)$.
For a symmetric definite positive matrix $S$, $\normin{x}_{S} = \sqrt{\Tr x^\top S x}$.
\paragraph{Model}
Consider the multitask\footnote{Results simplify in the single task case, where $q = 1$, $\Beta = \beta \in \bbR^{p}$, $\normin{\cdot}_{2,1}=\normin{\cdot}_1$, $\normin{\cdot}_{2,\infty}=\normin{\cdot}_\infty$.
We state these simpler results in \Cref{app_sec:single_case}.
}
linear regression model:
\begin{model}\label{eq:linear_model}
  Y = X \Beta^{*} + \Epsilon \enspace ,
\end{model}
where $Y \in \bbR^{n \times q}$, $X \in \bbR^{n \times p}$ is the deterministic design  matrix, $\Beta^* \in \bbR^{p \times q}$ are the true regression coefficients and $\Epsilon\in \bbR^{n \times q}$ models a centered noise.

%

For an estimator $\Betahat$ of $\Beta^*$, we aim at controlling $\normin{\Betahat-\Beta^*}_{2,\infty}$ with high probability,
and showing support recovery guarantees provided the non-zero coefficients are large enough.
To prove such results, the following assumptions are classical: Gaussianity and independence of the noise,
and \emph{mutual incoherence}.
\begin{assumption}\label{assum:gauss_noise}
    The entries of $\Epsilon$ are \iid $\cN(0, {\sigma^{*}}^2)$ random variables.
\end{assumption}
\begin{assumption}[Mutual incoherence] \label{assum:mut_inco}
    The \emph{Gram matrix} \mbox{$\Psi \eqdef \frac{1}{n} X^{\top} X$} satisfies
    \begin{align}
        \Psi_{j j}=1 \enspace,
        \text{ and }
        \max _{j' \neq j} \left|\Psi_{j j'}\right| \leq \tfrac{1}{7 \alpha s}, \, \forall j \in [p]
        \enspace ,
    \end{align}
    for some integer $s  \geq 1$ and some constant $\alpha > 1$.
\end{assumption}
Mutual incoherence of the design matrix (\Cref{assum:mut_inco}) implies the Restricted Eigenvalue Property introduced by \citet{Bickel_Ritov_Tsybakov09}.
\begin{lemma}[{Restricted Eigenvalue Property, \citet[Lemma 2]{Lounici08}}]\label{lem:rep}
    If \Cref{assum:mut_inco} is satisfied, then:
    \begin{align}
         \min_{\substack{\cS \subset [p]\\|\cS| \leq s}} \, \,
         \min_{\substack{\Delta \neq 0 \\ \norm{\Delta_{\cS^c}}_{2, 1} \leq
         3 \norm{\Delta_{\cS}}_{2, 1}}}
         \frac{1}{\sqrt{n}}\frac{\norm{X \Delta}_F}{\norm{\Delta_{\cS}}_F}
         \geq \sqrt{1-\tfrac{1}{\alpha}}
         >0 \enspace.
    \end{align}
In particular, with the choice
    $\Delta \eqdef \Betahat - \Betastar$, if $\norm{\Delta_{\cS_*^c}}_{2,1}
    \leq
    3 \norm{\Delta_{\cS_*}}_{2, 1}$, the following bound holds:
\begin{align}\label{eq:sparse_conditionning}
    \frac{1}{n} \norm{X \Delta}_F^2
    \geq
    \left(1 - \frac{1}{\alpha}\right)\norm{\Delta_{\cS_*}}_F^2
    \enspace.
 \end{align}
\end{lemma}
%
\subsection{Motivation and general proof structure}
\label{sub:motivation_structure}
%
\paragraph{Structure of all proofs}
We prove results of the following form for several estimators $\hat \Beta$ (summarized in \Cref{tab:summary_estimators}): for some parameter $\lambda$ independent of the noise level $\sigma^*$, with high probability,
\begin{align}
  \frac{1}{q} \normin{\hat \Beta - \Beta^*}_{2, \infty}
      \leq  C  \frac{1}{\sqrt{nq}} \sqrt{\frac{\log p}{q}} \sigma^*
       \enspace .
\end{align}
Then, assuming a signal strong enough such that
\begin{equation}
  \min_{j \in \cS^*} \tfrac{1}{q} \normin{\Beta_{j:}^*}_{2}
  >
  2 C  \frac{1}{\sqrt{nq}} \sqrt{\frac{\log p}{q}} \sigma^*
  \enspace,
\end{equation}
on the same event, for some $\eta > 0$, 
\begin{equation}
  \hat \cS \eqdef \{
    j \in [p] :
    \tfrac{1}{q} \normin{\hat \Beta_{j:}}_{2}
      > C (3 + \eta) \lambda \sigma^*
  \} \enspace
\end{equation}
matches the true sparsity pattern:
  $\hat \cS = \cS^*$  .

We explain here the general sketch proofs for all the estimators.
We assume that \Cref{assum:mut_inco} holds and then place ourselves on an event $\cA$ such that $\normin{X^\top Z}_{2, \infty} \leq \lambda / 2$ (for a $Z \in \partial f (\Epsilon)$, where $f$ is the datafitting term) in order to use \Cref{lem:bound_delta_psidelta}, which links the control of $\normin{\Psi (\hat \Beta - \Beta^*)}_{2, \infty}$ to the control of $\normin{ \hat \Beta - \Beta^*}_{2, \infty}$.
To obtain sup-norm convergence it remains for each estimator to:
\begin{itemize}[itemsep=0pt,partopsep=0pt,topsep=0pt,partopsep=0pt]
  \item control the probability of the event $\cA$ with classical concentration inequalities.
  \item control the quantity $\normin{\Psi (\hat \Beta - \Beta^*)}_{2, \infty}$, with:
  \begin{itemize}
      \item first order optimality conditions, which provide a bound on $\normin{X^\top Z}_{2, \infty}$: $\normin{X^\top \hat Z}_{2, \infty} \leq \lambda$ for a $\hat Z \in \partial f(\hat \Epsilon)$,
      \item the definition of the event $\cA$,
      \item
      for some estimators, an additional assumption (\Cref{assum:high_noise}).
  \end{itemize}
\end{itemize}
Next, we detail the lemmas used in this strategy.
\subsection{Preliminary lemma}

%
We now provide conditions leading to $\normin{\Delta_{\cS_*^c}}_{2, 1} \leq  3 \norm{\Delta_{\cS_*}}_{2, 1}$, to be able to apply \Cref{lem:rep}.
In this section we consider estimators of the form
\begin{problem} \label{pb:generic_optim_pb}
    \Betahat
    \eqdef
    \argmin_{\Beta \in \bbR^{p \times q}}
    f(Y - X \Beta) + \lambda \norm{\Beta}_{2, 1}
    \enspace ,
\end{problem}
for a proper, lower semi-continuous and convex function $f: \bbR^{n \times q} \to \bbR$ (see the summary in \Cref{tab:summary_estimators}).

{\centering
\begin{table*}[t]
  \caption{Summary of estimators (MT: multitask, MV: multivariate)}
  \label{tab:summary_estimators}
  \centering
  \begin{tabular}{lccccc}
    \toprule
     Name
     & $f(\Epsilon)$
     & Sup-norm cvg
     & Pred. cvg\\
    \midrule
    MT \sqrtlasso \eqref{pb:multitask_sqrt}
    & $\tfrac{1}{\sqrt{nq}}\normin{\Epsilon}_F$
    & \citet{Bunea_Lederer_She13}
    & \citet{Bunea_Lederer_She13} \\
    MT concomitant Lasso
    &  $\displaystyle \min_{\sigma > 0}
    \tfrac{1}{2nq\sigma}\normin{\Epsilon}_F^2 + \tfrac{\sigma}{2}$
    & us
    & \citet{Li_Haupt_Arora_Liu_Hong_Zhao16b}
     \\
    MT smooth. conco. Lasso \eqref{pb:multitask_smoothed_sqrt}
    &  $\displaystyle \min_{\sigma > \sigmamin}
    \tfrac{1}{2nq\sigma}\normin{\Epsilon}_F^2 + \tfrac{\sigma}{2}$
    & us
    & \citet{Li_Haupt_Arora_Liu_Hong_Zhao16b}
    \\
    \midrule
    MV \sqrtlasso \eqref{pb:multivar_sqrt}
    & $\tfrac{1}{n} \normin{\Epsilon / \sqrt{q}}_*$
    & us
    & \citet{Molstad19}
    \\
    MV conco. \sqrtlasso \eqref{pb:conco_multivar}
    & $\displaystyle \min_{S \succ 0}
      \tfrac{1}{2nq}\normin{\Epsilon}_{S^{-1}}^2 + \tfrac{1}{2n}\Tr(S)$
    & us
    & \citet{Molstad19}
    \\
    MV SGCL \eqref{pb:smoothed_multivar_sqrt}
    & $\displaystyle \min_{\sigmamax \succeq S \succeq \sigmamin}
      \tfrac{1}{2nq}\normin{\Epsilon}_{S^{-1}}^2 + \tfrac{1}{2n}\Tr(S)$
    & us
    &
    \\
    \bottomrule
  \end{tabular}
\end{table*}
}
Fermat's rule for \Cref{pb:generic_optim_pb} reads:
\begin{align}\label{eq:kkt}
     0 \in X^\top \partial f (\hat \Epsilon)
     + \lambda  \partial \normin{\cdot}_{2, 1}(\Betahat) \enspace,
\end{align}
Hence, we can find $\hat Z \in \partial f (\hat \Epsilon)$ such that
\begin{equation} \label{eq:generic_danztig_constraint}
  \normin{X^\top \hat Z}_{2, \infty} \leq \lambda \enspace.
\end{equation}
\begin{restatable}{lemma}{lemmacompsupport}
    Consider an estimator based on \Cref{pb:generic_optim_pb}, and assume that there exists $Z \in \partial f (\Epsilon)$ such that $\normin{X^\top Z}_{2, \infty} \leq \lambda / 2$.
    Then:
    \begin{lemmaenum}[partopsep=0pt,topsep=0pt,parsep=0pt]
        \item $\norm{\Delta_{\cS_*^c}}_{2, 1} \leq 3 \norm{\Delta_{\cS_*}}_{2, 1} \enspace,$  \label{lem:bds_S_c_S}
        \item \label{lem:bound_delta_psidelta}
          if $\Psi$ and $\alpha$ satisfy \Cref{assum:mut_inco},
          \vspace{-3mm}
          \begin{equation}
          \norm{\Delta}_{2, \infty}  \leq
          \left(1 + \frac{16}{7(\alpha - 1) } \right) \norm{\Psi \Delta}_{2, \infty} \enspace .\nonumber
         \end{equation}
    \end{lemmaenum}
\end{restatable}

\begin{proof} For \Cref{lem:bds_S_c_S}, we use the minimality of $\Betahat$:
  \begin{align*}
    f(\hat\Epsilon) - f(\Epsilon)
    & \leq
     \lambda \normin{\Betastar}_{2, 1} - \lambda \normin{\Betahat}_{2, 1} \enspace .
     \numberthis
    \label{eq:general_minimality}
\end{align*}
  We upper bound the right hand side of \Cref{eq:general_minimality}, using
  $\normin{\Betahat}_{2, 1} =  \normin{\Betahat_{\cS_*}}_{2, 1} + \normin{\Betahat_{\cS_*^c}}_{2, 1}$,
  $\Betastar_{{\cS_*^c}} = 0$ and with the triangle inequality:
  \begin{align*}
    \normin{\Betastar}_{2, 1} - \normin{\Betahat}_{2, 1}
    &=
    \normin{\Betastar_{\cS_*}}_{2, 1}
    - \normin{\Betahat_{\cS_*}}_{2, 1} - \normin{\Betahat_{{\cS_*^c}}}_{2, 1} \\
    &=
    \normin{\Betastar_{\cS_*}}_{2, 1}
    - \normin{\Betahat_{\cS_*}}_{2, 1}
    - \normin{\Delta_{{\cS_*^c}}}_{2, 1}
     \\
    &\leq
    \normin{(\Betastar - \Betahat)_{\cS_*}}_{2, 1}
    - \normin{\Delta_{\cS_*^c}}_{2, 1}   \\
    & \leq \normin{\Delta_{\cS_*}}_{2, 1}
    - \normin{\Delta_{\cS_*^c}}_{2, 1} \numberthis
    \enspace . \label{eq:general_upper_bound}
\end{align*}
  We now aim at finding a lower bound of the left hand side of \Cref{eq:general_minimality}.
  By convexity of $f$, $\partial f(\Epsilon) \neq \varnothing$.
  Picking $Z \in \partial f(\Epsilon)$ such that $\normin{X^\top Z}_{2, \infty} \leq \frac{\lambda}{2}$ yields:
  \begin{align*}
      f(Y - X \Betahat) -  f(Y - &X \Beta^*)
    \geq
    - \left \langle Z, X (\Betahat - \Betastar)
    \right \rangle
    \enspace \\
    &\geq
      - \left \langle X^\top Z,  \Delta \right \rangle  \\
    &\geq
      - \norm{X^\top  Z}_{2, \infty} \norm{\Delta}_{2, 1}
    \\
    &\geq
    - \frac{1}{2} \lambda \norm{\Delta}_{2, 1}
    \enspace . \label{eq:general_lower_bound}
\end{align*}
  Combining \Cref{eq:general_minimality,eq:general_upper_bound,eq:general_lower_bound} leads to:
  \begin{align}
    - \frac{1}{2} \norm{\Delta}_{2, 1}
    &\leq \normin{\Delta_{\cS_*}}_{2, 1} - \normin{\Delta_{\cS_*^c}}_{2, 1}
    \nonumber\\
    \norm{\Delta_{\cS_*^c}}_{2, 1}
    &\leq
     3 \norm{\Delta_{\cS_*}}_{2, 1}
    \enspace .
  \end{align}
Proof of \Cref{lem:bound_delta_psidelta} is a direct application of \Cref{lem:bnd_l_infty_l_infty,lem:bds_S_c_S}.
\end{proof}
Equipped with these Assumptions and Lemmas, we will show that the considered estimators reach the minimimax lower bounds, which we recall in the following.
\subsection{Minimax lower bounds}
%
As said in \Cref{sub:motivation_structure}, our goal is to provide convergence rates on the quantity $\normin{\hat \Beta - \Beta^*}_{2, \infty}$.
To show that our bounds are ``optimal'' we recall that the considered estimators achieve minimax rate (up to a logarithmic factor).
Indeed, under some additional assumptions controlling the conditioning of the design matrix, one can show \citep{Lounici_Pontil_vandeGeer_Tsybakov11} minimax lower bounds.
%
\begin{assumption}\label{assum:design_matrix_X}
  For all $\Delta \in \bbR^{p \times q} \backslash \{ 0 \}$ such that $|\cS(\Delta) | \leq  2 | \cS^* |$:
  \begin{equation}
    \kappamin \leq
    \frac{\norm{X \Delta}_F^2}{n \norm{\Delta}_F^2} \leq \kappamax \enspace .
  \end{equation}
\end{assumption}
Provided \Cref{assum:gauss_noise,assum:design_matrix_X} hold true, \citet[Thm. 6.1]{Lounici_Pontil_vandeGeer_Tsybakov11} proved the following minimax lower bound (with an absolute constant $R$):
\begin{align}
  \inf_{\hat \Beta} \sup_{\substack{\Beta^*  \st \\
  |\cS(\Beta^*)| \leq s }} \bbE \Big ( \tfrac{1}{q} \normin{\hat \Beta
  - \Beta^*}_{2, \infty} \Big )
  \geq
   \tfrac{R \sigma^* }{\kappamax \sqrt{n}}
   \sqrt{1 + \tfrac{\log(ep/s)}{q}}  \enspace . \nonumber
\end{align}
%
\subsection{Smoothing}
\label{sub:smoothing}
Some of the pivotal estimators studied here are obtained via a technique called smoothing.
For $L > 0$, a convex function $\phi$ is $L$-smooth (\ie its gradient is $L$-Lipschitz) if and only if its Fenchel conjugate $\phi^*$ is $\frac 1L$-strongly convex \citep[Thm 4.2.1]{Hiriart-Urruty_Lemarechal93b}.
Therefore, given a smooth function $\omega$, a principled way to smooth a function $f$ is to add the strongly convex $\omega^*$ to $f^*$, thus creating a strongly convex function, whose Fenchel transform is a smooth approximation of $f$.
Formally, given a smooth convex function $\omega$, the $\omega$-smoothing of $f$ is $(f^* + \omega^*)^*$.
By properties of the Fenchel transform, the latter is also equal to $f \infconv \omega$ whenever $f$ is convex \citep[Prop. 13.21]{Bauschke_Combettes11}.

\begin{proposition}\label{prop:smooth_fro}
    Let $\omega_{\sigmamin} = \frac{1}{2 \sigmamin} \normin{\cdot}_F^2 + \frac{\sigmamin}{2}$.
    The $\omega_{\sigmamin}$-smoothing of the Frobenius norm is equal to:
    \begin{align}\label{eq:smooth_fro}
        \left(\omega_{\sigmamin} \infconv \normin{\cdot}_F\right) (Z)&=
        \begin{cases}
            \normin{Z}_F \enspace, \text{ if } \normin{Z}_F \leq \sigmamin \enspace, \\
            \frac{1}{2 \sigmamin} \normin{Z}_F^2 + \frac{\sigmamin}{2} \enspace, \text{ if } \normin{Z}_F \geq \sigmamin \enspace.
        \end{cases} \nonumber
        \\
        &= \min_{\sigma \geq \sigmamin} \frac{1}{2\sigma} \norm{Z}_F^2 + \frac{\sigma}{2} \enspace.
    \end{align}
\end{proposition}

\section{Multitask square-root Lasso}
\label{sec:multitask_sqrt}
%
It is clear that the multitask square-root Lasso (\cref{pb:multitask_sqrt}) suffers from the same numerical weaknesses as the square-root Lasso.
A more amenable version has been introduced by \citet[Prop. 21]{Bertrand_Massias_Gramfort_Salmon19a}.
The smoothed multitask square-root Lasso is obtained by replacing the non-smooth function $\normin{\cdot}_F$ with a smooth approximation, depending on a parameter $\sigmamin > 0$:
\begin{problem}\label{pb:multitask_smoothed_sqrt}
    \argmin_{\Beta \in \bbR^{p\times q}}
     \left ( \norm{\cdot}_F \infconv \big(\tfrac{1}{2 \sigmamin} \normin{\cdot}^2 + \tfrac{\sigmamin}{2} \big) \right )
     \left (\tfrac{Y - X \Beta}{\sqrt{nq}} \right )
    + \lambda \norm{\Beta}_{2,1}
    \enspace .
\end{problem}
Plugging the expression of the smoothed Frobenius norm \eqref{eq:smooth_fro}, the problem formulation becomes:
\begin{problem}
    (\Betahat, \hat \sigma) \in \argmin_{\substack{\Beta \in \bbR^{p \times q} \\ \sigma \geq \sigmamin}}
    \frac{1}{2 n q \sigma} \normin{Y - X \Beta}_F^2 + \frac{\sigma}{2} + \lambda \norm{\Beta}_{2,1} \enspace,
\end{problem}
where the datafitting term is $(n q \sigmamin)^{-1}$-smooth \wrt $\Beta$. 
We show that estimators \eqref{pb:multitask_sqrt} and \eqref{pb:multitask_smoothed_sqrt} reach the minimax lower bound, with a regularization parameter independent of $\sigma^*$.
For that, another assumption is needed.
%
\begin{assumption}[{\citet[Lemma 3.1]{vandeGeer16}}] \label{assum:high_noise}
  There exists $\eta > 0$ verifying
  \begin{align}
    \lambda \normin{\Betastar}_{2, 1}
    \leq
    \eta \sigma^* \enspace .
  \end{align}
\end{assumption}
%
\begin{proposition}\label{prop:sqrtmtl_bds_est}
    Let $\hat \Beta$ denote the multitask square-root Lasso \eqref{pb:multitask_sqrt} or its smoothed version \eqref{pb:multitask_smoothed_sqrt}.
    Let \Cref{assum:gauss_noise} be satisfied, let $\alpha$ and $\eta$ satisfy \Cref{assum:mut_inco,assum:high_noise}.
    For $C = \big(1 + \tfrac{16}{7(\alpha - 1)}\big)$, $A > \sqrt{2}$ and
    $\lambda = \frac{2\sqrt{2}}{ \sqrt{nq}} \big (1 + A \sqrt{(\log p) / {q} } \big )$,
    if $\sigmamin \leq \tfrac{\sigma^*}{\sqrt{2}}$ then with probability at least $1 - p^{1 - A^2/2} - (1 + e^2) e^{-nq/24}$,
    \begin{equation}
        \tfrac{1}{q} \normin{\hat \Beta - \Beta^*}_{2, \infty}
            \leq  C (3 + \eta) \lambda \sigma^*
             \enspace .
    \end{equation}
    %
    Moreover provided that
    \begin{equation}
      \min_{j \in \cS^*} \tfrac{1}{q} \normin{\Beta_{j:}^*}_{2}
      >
      2 C (3 + \eta) \lambda \sigma^* \enspace,
    \end{equation}
    then, with the same probability, the estimated support
    \begin{equation}
      \hat \cS \eqdef \{
        j \in [p] :
        \tfrac{1}{q} \normin{\hat \Beta_{j:}}_{2}
          > C (3 + \eta) \lambda \sigma^*
      \} \enspace
    \end{equation}
    recovers the true sparsity pattern: $\hat \cS = \cS^*$.
\end{proposition}
\begin{proof}
    We first bound $\normin{\Psi \Delta}_{2, \infty}$.
    Let $\cA_1$ be the event
    \begin{equation}
      \cA_1 \eqdef \left \{ \tfrac{\normin{X^\top \Epsilon}_{2, \infty}}{\sqrt{nq}\normin{\Epsilon}_F} \leq \tfrac{\lambda}{2}   \right \}
      \cap \left \{ \tfrac{\sigma^*}{\sqrt{2}} < \tfrac{\normin{\Epsilon}_F}{\sqrt{nq}} < 2 \sigma^* \right \}
      \enspace .
    \end{equation}
    By \Cref{lem:control_A1}, $\bbP(\cA_1) \geq 1 - p^{1 - A^2/2} - (1 + e^2) e^{-nq/24} $.
    For both estimators, on $\cA_1$ we have:
      \begin{align}
        n\normin{\Psi \Delta}_{2, \infty}
        &=
        \normin{X^\top(\hat \Epsilon - \Epsilon)}_{2, \infty} \nonumber \\
        & \leq
         \normin{X^\top \hat \Epsilon}_{2, \infty}
        + \normin{X^\top \Epsilon}_{2, \infty} \nonumber \\
        & \leq
         \normin{X^\top \hat \Epsilon}_{2, \infty}
        +  \lambda n q \sigma^* \enspace ,\label{eq:bound_n_psi_delta}
      \end{align}
  hence we need to bound $\normin{X^\top \hat \Epsilon}_{2, \infty}$.
  We do so using optimality conditions, that yield for \Cref{pb:multitask_sqrt}, with $\hat \Epsilon \neq 0$,
  \begin{align}
      \normin{X^\top \tfrac{\hat \Epsilon}{\normin{\hat \Epsilon}}_{F}}_{2, \infty}
      &\leq \lambda \sqrt{nq} \nonumber\\
      \tfrac{1}{nq} \normin{X^\top \hat \Epsilon}_{2, \infty}
      &\leq \lambda \tfrac{\normin{\hat \Epsilon}_F}{\sqrt{nq}} \enspace, \label{eq:kkt_multitask_sqrt}
  \end{align}
  and the last equation is still valid if $\hat \Epsilon = 0$.
  For \Cref{pb:multitask_smoothed_sqrt}, the optimality conditions yield:
  \begin{equation}
    \begin{cases}
      \tfrac{1}{nq} \normin{X^\top \hat \Epsilon}_{2, \infty}
      \leq
      \lambda  \tfrac{\normin{\hat \Epsilon}_F}{\sqrt{nq}} \enspace,
      & \text{ if } \tfrac{\normin{\hat \Epsilon}_F}{\sqrt{nq}} \geq \sigmamin  \enspace,\\
      \tfrac{1}{nq} \normin{X^\top \hat \Epsilon}_{2, \infty}
      \leq
      \lambda \sigmamin \enspace,
      & \text{otherwise} \enspace.
    \end{cases}
  \end{equation}
  Therefore,
  \begin{equation}\label{eq:kkt_multitask_smoothed_sqrt}
      \tfrac{1}{nq} \normin{X^\top \hat \Epsilon}_{2, \infty}
      \leq
      \lambda \max \left (\tfrac{\normin{\hat \Epsilon}_F}{\sqrt{nq}}, \sigmamin \right ) \enspace .
  \end{equation}
  It now remains to bound $\normin{\hat \Epsilon}_F$ for both estimators, which is done with \Cref{assum:high_noise}:
  for \Cref{pb:multitask_sqrt}, by minimality of the estimator,
\begin{align}
  \tfrac{1}{\sqrt{nq}} \normin{\hat \Epsilon}_F
  + \lambda \normin{\hat \Beta}_{2, 1}
  &\leq
  \tfrac{1}{\sqrt{nq}} \normin{\Epsilon}_F
  + \lambda \normin{\Beta^*}_{2, 1} \nonumber \\
  \tfrac{1}{\sqrt{nq}} \normin{\hat \Epsilon}_F
  &\leq
  \tfrac{1}{\sqrt{nq}} \normin{\Epsilon}_F
  + \lambda \normin{\Beta^*}_{2, 1} \nonumber \\
  & \leq 2\sigma^* + (1 + \eta) \sigma^*  \nonumber \\
  & \leq (3 + \eta ) \sigma^*  \enspace, \label{eq:link_f_sigma}
\end{align}
and we can obtain the same bound in the case of \Cref{pb:multitask_smoothed_sqrt} (see \Cref{lem:link_hat_Epsilon_smooth_sqrt}).
%
  Combining \Cref{eq:bound_n_psi_delta,eq:kkt_multitask_sqrt,eq:kkt_multitask_smoothed_sqrt,eq:link_f_sigma} we have in both cases:
  \begin{align} \label{eq:bound_1_q_psi_delta}
    \tfrac{1}{q} \normin{\Psi \Delta}_{2, \infty} \leq (3 + \eta) \lambda \sigma^* \enspace .
  \end{align}
  Finally we exhibit an element of $\partial f(\Epsilon)$ to apply \Cref{lem:bound_delta_psidelta}.
  Recall that  $f = \frac{1}{\sqrt{nq}} \normin{\cdot}_F$ for \Cref{pb:multitask_sqrt},
  and $f = \norm{\cdot}_F \infconv \big(\tfrac{1}{2 \sigmamin} \normin{\cdot}_F^2 + \tfrac{\sigmamin}{2} \big) (\frac{\cdot}{\sqrt{nq}})$ for \Cref{pb:multitask_smoothed_sqrt}.
  On $\cA_1$, $\partial f (\Epsilon)$ is a singleton for both estimators, whose element is  $\Epsilon / (\normin{\Epsilon}_F \sqrt{nq})$.

  Additionally, on $\cA_1$ the inequality $\frac{1}{\sqrt{nq}} \frac{\normin{X^\top \Epsilon}_{2, \infty}}{ \normin{\Epsilon}_F} \leq \frac{\lambda}{2}$ holds,
  meaning we can apply \Cref{lem:bound_delta_psidelta} with $Z = \Epsilon / (\normin{\Epsilon}_F \sqrt{nq}) $.
  This proves the bound on $\normin{\Delta}_{2, \infty}$.
  Then, the support recovery property easily follows from \citet[Cor. 4.1]{Lounici_Pontil_Tsybakov_vandeGeer09}.
\end{proof}
\paragraph{Single task case} For the purpose of generality, we proved convergence results for the multitask versions of the square-root/concomitant Lasso and its smoothed version, but the results are also new in the single task setting.
Refined bounds of \Cref{prop:sqrtmtl_bds_est} in the single-task case are in \Cref{app_sub:sqrt_lasso_single_task}.

%
\section{Multivariate square-root Lasso}
\label{sec:multivar_sqrt}
%
Here we show that the multivariate square-root Lasso\footnote{we keep the name of \cite{vandeGeer16}, although a better name in our opinion would be the (multitask) trace norm Lasso, but the name is used by \cite{Grave_Obozinski_Bach11} when the nuclear norm is used as a regularizer} and its smoothed version also reach the minimax rate.
Recall that the multivariate square-root Lasso is \Cref{pb:multivar_sqrt}.
For the numerical reasons mentioned above, as well as to get rid of the invertibility assumption of $\hat \Epsilon^\top \hat \Epsilon$, we consider the smoothed estimator of \citet{Massias_Fercoq_Gramfort_Salmon17}:
\begin{problem}\label{pb:smoothed_multivar_sqrt}
    \argmin_{\substack{\Beta \in \bbR^{p \times q} \\
            \sigmamax \Id_n \succeq \Snoise \succeq \sigmamin \Id_n }}
    \frac{1}{2nq} \normin{Y - X \Beta}_{\Snoise^{-1}}^2
    + \frac{\Tr S}{2n}
    + \lambda \norm{\Beta}_{2,1}
    \enspace .
\end{problem}
The variable introduced by concomitant formulation is now a matrix $S$, corresponding to the square root of the noise covariance estimate.
The multivariate square-root Lasso~\eqref{pb:multivar_sqrt} and its concomitant formulation~\eqref{pb:conco_multivar} have the same solution in $\Beta$ provided $\hat \Epsilon^\top \hat \Epsilon$ is invertible.
In this case, the solution of \Cref{pb:conco_multivar} in $S$ is $\hat S = ( \tfrac 1q \hat \Epsilon \hat \Epsilon^\top)^{\frac{1}{2}}$.

\Cref{pb:smoothed_multivar_sqrt} is actually a small modification of \cite{Massias_Fercoq_Gramfort_Salmon17}, where we have added the second constraint $S \preceq \sigmamax \Id_n$.
$\sigmamax$ can for example be set as $\normin{(\tfrac 1q Y Y^\top)^{1/2}}_2$, as \Cref{fig:sing_val_residuals} illustrates that this is the order of magnitude of $\normin{\hat S}_2$.
Because of these constraints, the solution in $S$ is different from that of \Cref{pb:conco_multivar}.
We write a singular value decomposition of $\frac{1}{\sqrt{q}} \hat \Epsilon$: $UDV^\top$, with
$D = \diag(\gamma_i) \in \bbR^{n\times n}$, $U \in \bbR^{n \times n}$ and $V \in \bbR^{q \times n}$ such that $U^\top U = V^\top V = \Id_n$.
Then the solution in $S$ to \Cref{pb:smoothed_multivar_sqrt} is $\hat \Snoise = U \diag \big ( [\gamma_i]_{\sigmamin}^{\sigmamax} \big ) U^\top$ (this result is easy to derive from \citet[Prop. 2]{Massias_Fercoq_Gramfort_Salmon17}).
$\hat S$ can be used to bound $\normin{X^\top \hat \Epsilon}_{2, \infty}$:
%
\begin{lemma}\label{lem:bounds_res} (Proof in \Cref{lem:bound:residuals_sgcl})
  For the concomitant multivariate square-root Lasso~\eqref{pb:conco_multivar} and the smoothed concomitant multivariate square-root~\eqref{pb:smoothed_multivar_sqrt} we have:
  \begin{align*}
    \normin{X^\top \hat \Epsilon}_{2, \infty}
    &\leq \normin{\hat S }_2 \normin{ X^\top \hat S^{-1}\hat \Epsilon}_{2, \infty} \numberthis \label{eq:bound:XEpsilon_conco_multivar}\enspace.
  \end{align*}
\end{lemma}
We can prove the minimax sup-norm convergence of these two estimators, using the following assumptions.
\begin{assumption}\label{assum:bound_residuals}
  For the multivariate square-root Lasso, $\hat \Epsilon^\top \hat \Epsilon$ is invertible, and there exists $\eta$ such that
  $\normin{( \tfrac{1}{q} \hat \Epsilon^\top \hat \Epsilon)^{\frac{1}{2}}}_2 \leq  (2 + \eta)\sigma^*$.
\end{assumption}
We get rid of this very strong hypothesis for the smoothed version, as the estimated noise covariance is invertible because of the constraint $S \succeq \sigmamin \Id_n$, and we can control its operator norm via the constraint $S \preceq \sigmamax \Id_n$. We still need an assumption on $\sigmamin$ and $\sigmamax$.
\begin{assumption}\label{assum:smoothing_param}
    $\sigmamin$, $\sigmamax$ and $\eta$ verify: $\sigmamin \leq \tfrac{\sigma^*}{\sqrt{2}}$ and $\sigmamax = (2 + \eta) \sigma^*$ with $\eta \geq 1$.
\end{assumption}
%
\begin{proposition}\label{prop:multivariate_sqrt_bds_est}
  For the multivariate square-root Lasso \eqref{pb:multivar_sqrt} (\resp its smoothed version \eqref{pb:smoothed_multivar_sqrt}), let \Cref{assum:gauss_noise} be satisfied, let $\alpha$ satisfy \Cref{assum:mut_inco} and let $\eta$ satisfy \Cref{assum:bound_residuals} (\resp let $\sigmamin, \sigmamax, \eta$ satisfy \Cref{assum:smoothing_param}).
  Let $C = \big(1 + \tfrac{16}{7(\alpha - 1)}\big)$, $A \geq \sqrt{2}$,
  and $\lambda = \frac{2\sqrt{2}}{ \sqrt{nq}} (1 + A \sqrt{(\log p) / q})$.
  Then there exists $c \geq 1/64$ such that with probability at least $1 - p^{1 - A^2/2} - 2 n e^{-cq/n}$,
  \begin{equation}
      \tfrac{1}{q} \normin{\hat \Beta - \Beta^*}_{2, \infty}
      \leq C (3 + \eta) \lambda \sigma^*
       \enspace .
  \end{equation}
  Moreover if
  \begin{equation}
    \min_{j \in \cS^*} \tfrac{1}{q} \normin{\Beta_{j:}^*}_{2}
    >
    2 C (3 + \eta) \lambda \sigma^* \enspace,
  \end{equation}
  then with the same probability:
  \begin{equation}
    \hat \cS \eqdef \{
      j \in [p] :
      \tfrac{1}{q} \normin{\hat \Beta_{j:}}_{2}
        > C (3 + \eta) \lambda \sigma^*
    \} \enspace
  \end{equation}
  correctly estimates the true sparsity pattern:
    $\hat \cS = \cS^*$.
\end{proposition}
\begin{proof}
  Let $\cA_2$ be the event:
  \begin{equation}
    \Big\{ \tfrac{\normin{X^\top \Epsilon  }_{2, \infty}}{nq} \leq \tfrac{\lambda \sigma^*}{2 \sqrt{2}} \Big\}
  \cap \{ 2 \sigma^* \Id_q
    \succ (\tfrac{\Epsilon^\top \Epsilon}{n})^{\frac{1}{2}}
    \succ \tfrac{\sigma^*}{\sqrt{2}} \Id_q \}
    \enspace.
  \end{equation}
  By \Cref{lem:control_A2}, $\bbP (\cA_2) \geq 1 - p^{1 - A^2/2} - 2 n e^{-cq/n}$ ($c\leq 1/64)$.
  When the multivariate square-root Lasso residuals are full rank, the optimality conditions for \Cref{pb:multivar_sqrt,pb:smoothed_multivar_sqrt} read the same, but with differents $\hat S$ (introduced above):
  \begin{equation}\label{eq:kkt_multivar_sqrt}
    \normin{X^\top \hat S^{- 1} \hat \Epsilon}_{2, \infty}
    \leq
    \lambda q n
    \enspace .
\end{equation}
  With \Cref{lem:bounds_res,eq:kkt_multivar_sqrt} and \Cref{assum:bound_residuals}
  for the multivariate square-root Lasso (or \Cref{assum:smoothing_param} for its smoothed version):
  \begin{align}
    n \normin{\Psi & \Delta}_{2, \infty}
     = \normin{X^\top (\Epsilon - \hat \Epsilon)}_{2, \infty} \nonumber\\
    & \leq \normin{X^\top \hat \Epsilon}_{2, \infty} + \normin{X^\top \Epsilon}_{2, \infty}  \nonumber\\
    &\leq \lambda q n \normin{\hat S}_2
    + \normin{X^\top \Epsilon}_{2, \infty}  \nonumber\\
    & \leq \lambda (2 + \eta) q n \sigma^*
    + \normin{X^\top \Epsilon}_{2, \infty} \enspace .
  \end{align}
  Then on the event $\cA_2$:
  \begin{align}
    \frac{1}{q} \normin{\Psi \Delta}_{2, \infty}
    & \leq
    \lambda (2 + \eta) \sigma^*
    + \tfrac{1}{nq} \normin{X^\top \Epsilon}_{2, \infty} \nonumber \\
    & \leq
    \left (3 + \eta \right ) \lambda \sigma^* \enspace.
  \end{align}

Finally we exhibit an element of $\partial f(\Epsilon)$ to apply \Cref{lem:bound_delta_psidelta}.
Recall that  $f = \frac{1}{n\sqrt{q}} \normin{\cdot}_*$ for \Cref{pb:conco_multivar},
and $\displaystyle f = \min_{\substack{ \sigmamax \Id_n \succeq \Snoise \succeq \sigmamin \Id_n }}
\tfrac{1}{2nq} \normin{\cdot}_{\Snoise^{-1}}^2
+ \tfrac{\Tr S}{2n} $ for \Cref{pb:smoothed_multivar_sqrt}.
We also recall that for a full rank matrix $A \in \bbR^{n \times q}$ \citep[Sec. 2]{Koltchinskii_Lounici_Tsybakov11}:
\begin{align}
  \partial \normin{A}_* = \{
    (A A^\top)^{-1/2} A
    \} \enspace .
\end{align}
On $\cA_2$, $\partial f (\Epsilon)$ is a singleton for both estimators, whose element is  $(\Epsilon \Epsilon^\top)^{-1/2} \Epsilon / (n\sqrt{q}) $.
Additionally on $\cA_2$, using the same proof as in \Cref{lem:bound:residuals_sgcl}:
\begin{align}
  \tfrac{1}{n\sqrt{q}} \normin{X^\top (\Epsilon \Epsilon^\top)^{- 1 / 2} \Epsilon }_{2, \infty}
  & \leq \tfrac{1}{nq} \normin{X^\top \Epsilon}_{2, \infty}
        \normin{(\tfrac{\Epsilon \Epsilon^\top}{q})^{- 1 / 2}}_2 \nonumber \\
  & \leq \frac{\lambda \sigma^* }{2 \sqrt{2}} \times \frac{\sqrt{2}}{\sigma^*}
  \leq \frac{\lambda}{2} \enspace ,
\end{align}
meaning we can apply \Cref{lem:bound_delta_psidelta} with $Z= \Epsilon (\Epsilon^\top \Epsilon)^{-1/2}/ n\sqrt{q}$.
  This proves the bound on $\normin{\Delta}_{2, \infty}$.
  Then, the support recovery property easily follows from \citet[Cor. 4.1]{Lounici_Pontil_Tsybakov_vandeGeer09}.
\end{proof}



\section{Experiments}
\label{sec:experiments}
%
We first describe the setting of \Cref{fig:lambda_opt_noise,fig:sing_val_residuals}.
Then we show that empirically that results given by \Cref{prop:sqrtmtl_bds_est,prop:multivariate_sqrt_bds_est} hold in practice.
The signal-to-noise ratio (SNR) is defined as $\frac{\normin{X \Beta^*}_F}{\normin{Y - X \Beta}_F}$.

\subsection{Pivotality of the square-root Lasso}
\label{sub:lambda_opt_noise}
In this experiment the matrix $X$ consists of the $\num{10000}$ first columns of the \emph{climate} dataset ($n=864$).
We generate $\beta^*$ with $20$ non-zero entries.
Random Gaussian noise is added to $X \beta^*$ to create $y$, with a noise variance $\sigma^*$ controlling the SNR.

For each SNR value, both for the Lasso and the square-root Lasso, we compute the optimal $\lambda$ on a grid between $\lambda_{\max}$ (the estimator specific smallest regularization level yielding a 0 solution), using cross validation on prediction error on left out data.
For each SNR, results are averaged over 10 realizations of $y$.

\Cref{fig:lambda_opt_noise} shows that, in accordance with theory, the optimal $\lambda$ for the Lasso depends linearly on the noise level, while the square-root Lasso achieves pivotality.

\subsection{Rank deficiency experiment}
\label{sub:expe_rank_deficency}
%
For $(n, q, p) = (10, 20, 30)$, we simulate data: entries of $X$ are \iid $\cN(0, 1)$, $\Beta^*$ has $5$ non zeros rows, and Gaussian noise is to $X \Beta^*$ added to result in a SNR of 1.
We reformulate \Cref{pb:multivar_sqrt} as a Conic Program, and solve it with the SCS solver of \texttt{cvxpy} \citep{Odonoghue_Chu_Parikh_Boyd16,cvxpy} for various values of $\lambda$ ($\lambda_{\max}$ is the smallest regularization value yielding a null solution).
We then plot the singulars values of the residuals at optimum, shown on \Cref{fig:sing_val_residuals}.

Since the problem is reformulated as a Conic Program and solved approximately (precision $\epsilon=10^{-6}$), the residuals are not exact; however the sudden drop of singular values of $Y - X \Betahat$ must be interpreted as the singular value being exactly 0.
One can see that even for very high values of $\lambda$, the residuals are rank deficient while the matrix $Y$ is not.
This is most likely due to the trace penalty on $S$ in the equivalent formulation of \Cref{pb:conco_multivar}, encouraging singular values to be 0.
Therefore, even on simple toy data, the hypothesis used by \citet{vandeGeer_Stucky16,Molstad19} does not hold, justifying the need for smoothing approaches, both from practical and theoretical point of views.

\subsection{(Multitask) smoothed concomitant Lasso}
\label{sub:expes_scl}
%
%
\def \figsize {1}
\begin{figure}[t]
    \centering
    \begin{minipage}{0.8\linewidth}
        \includegraphics[width=\figsize\linewidth]{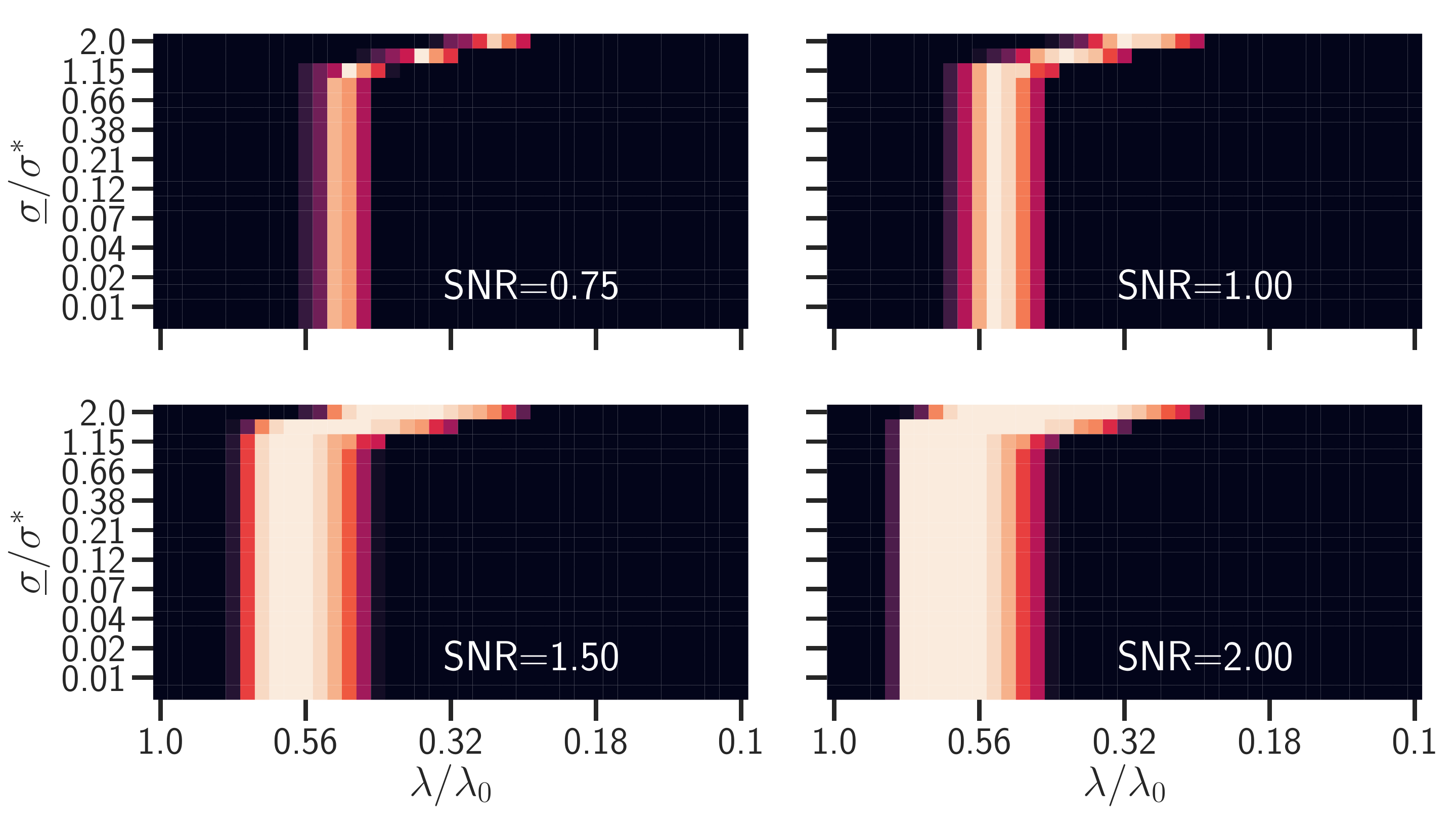}
    \end{minipage}
    \begin{minipage}{0.1\linewidth}
        \includegraphics[width=\figsize\linewidth]{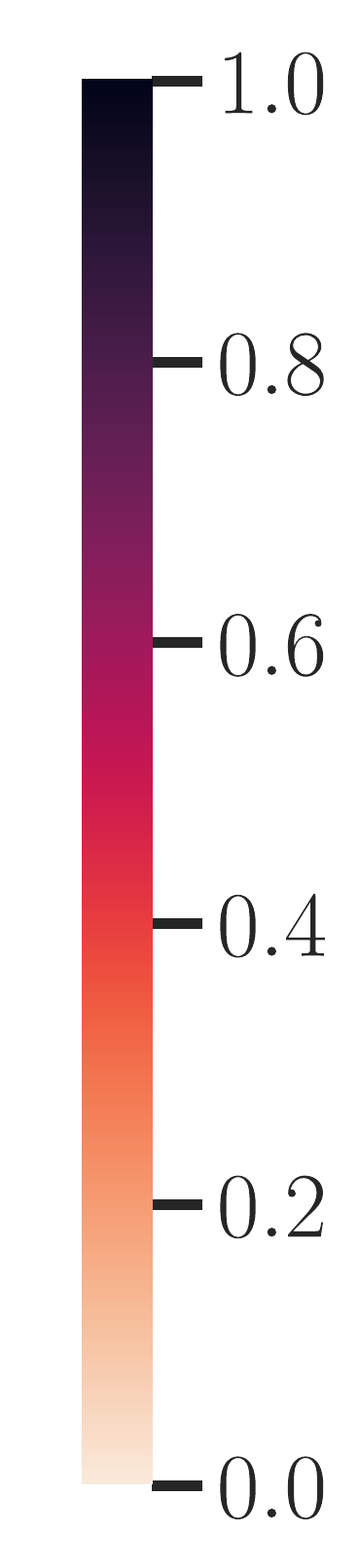}
    \end{minipage}
    \caption{(Synthetic data, $n=50$, $p=1000$, $q=20$) Hard recovery loss for different values of SNR for the multitask SCL.
    }
    \label{fig:hard_recov_SCL}
\end{figure}



\begin{figure}[t]
    \centering
    \begin{minipage}{0.8\linewidth}
        \includegraphics[width=\figsize\linewidth]{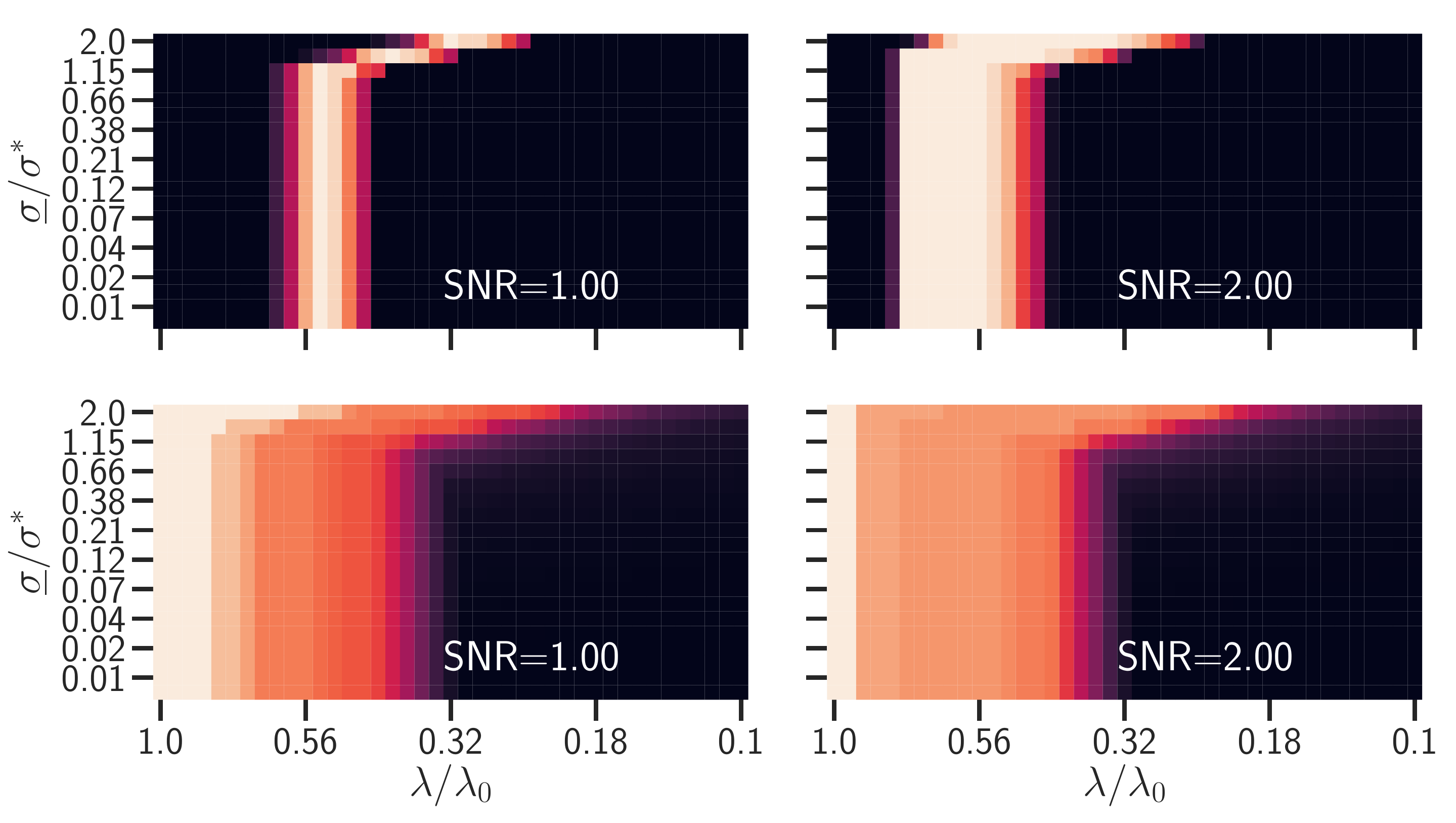}
    \end{minipage}
    \begin{minipage}{0.1\linewidth}
        \includegraphics[width=\figsize\linewidth]{colorbar_vert_hard_recov_SCL_rho_X_0_5}
    \end{minipage}
    \caption{(Synthetic data, $n=50$, $p=1000$, $q=20$) Hard recovery loss (top) and percent of non-zeros coefficients (bottom) for different values of SNR: $\SNR=1$ (left), $\SNR=2$ (right) for the multitask SCL.}
        \label{fig:hard_recov_size_supportSCL}
\end{figure}

Here we illustrate, as indicated by theory, that when the smoothing parameter $\sigmamin$ is sufficiently small, the multitask SCL is able to recover the true support (\Cref{prop:sqrtmtl_bds_est}).
More precisely, when $\sigmamin \leq \sigma^* / \sqrt{2}$, there exist a $\lambda$, independent of $\sigmamin$ and $\sigma^*$, such that the multitask SCL recovers the true support with high probability.
We use $(n, q, p) = (50, 50, 1000)$.
The design $X$ is random with Toeplitz-correlated features with parameter $\rho_X = 0.5$ (correlation between $X_{:i}$ and $X_{:j}$ is $\rho_X^{|i - j|}$), and its columns have unit Euclidean norm.
The true coefficient $\Beta^*$ has $5$ non-zeros rows whose entries are \iid $\cN(0, 1)$.
%
\paragraph{Comments on \Cref{fig:hard_recov_SCL,fig:hard_recov_size_supportSCL}}
The multitask SCL relies on two hyperparameters: the penalization coefficient $\lambda$ and the smoothing parameter $\sigmamin$, whose influence we study here.
The goal is to show empirically that when $\sigmamin \leq \sigma^* / \sqrt{2}$ the optimal $\lambda$ does not depend on the smoothing parameter $\sigmamin$.
We vary $\lambda$ and $\sigmamin$ on a grid: for each pair $(\lambda, \sigmamin)$ we solve the multitask SCL.
For each solution $\hat \Beta^{(\lambda, \sigmamin)}$ we then compute a metric, the hard recovery (\Cref{fig:hard_recov_SCL}) or the size of the support (\Cref{fig:hard_recov_size_supportSCL}).
The metrics are averaged over $100$ realizations of the noise.
\Cref{fig:hard_recov_SCL} shows the latter graph for different values of SNR.
We can see that when $\sigmamin \leq \sigma^*$, support recovery is achieved for $\lambda$ independent of $\sigmamin$.
As soon as $\sigmamin > \sigma^*$ the optimal $\lambda$ depends on $\sigmamin$.
When $\sigmamin$ reaches a large enough value (\ie $\sigma^*$) then the recovery profile is modified: the optimal $\lambda$ decreases as $\sigmamin$ grows.
This is logical, since as soon as the constraint is saturated, the (multitask) SCL boils down to a multitask Lasso with regularization parameter $\lambda \sigmamin$:
\begin{problem}\label{pb:saturated_conco}
    \Betahat
    \eqdef
    \argmin_{\Beta \in \bbR^{p\times q}}
    \frac{1}{2nq} \norm{Y - X \Beta}_F^2
    + \lambda \sigmamin \norm{\Beta}_{2,1}
    \enspace .
\end{problem}
\Cref{fig:hard_recov_size_supportSCL} shows that with a fixed $\lambda$ higher values of $\sigmamin$ may lead to smaller support size, see \eg $\lambda / \lambda_0 = 0.32$.

\subsection{Smoothed generalized concomitant Lasso (SGCL)}
\label{sub:expes_sgcl}
%
The experimental setting is the same as before, except here we used $(n, q, p) = (150, 100, 500)$.
\Cref{fig:hard_recov_SGCL} illustrates \Cref{prop:multivariate_sqrt_bds_est}.
When $\sigmamin \leq \sigma^*$, there exist a $\lambda$ that does not depend on $\sigmamin$ and such that SGCL finds the true support $\cS^*$.
However, as before, when $\sigmamin \geq \sqrt{2} \sigma^*$, $\lambda$ depends on $\sigmamin$.

%
%
\def \figsizetwo {0.8}
%
%
%
\begin{figure}[t]
    \centering
    \begin{minipage}{0.8\linewidth}
        \includegraphics[width=\figsize\linewidth]{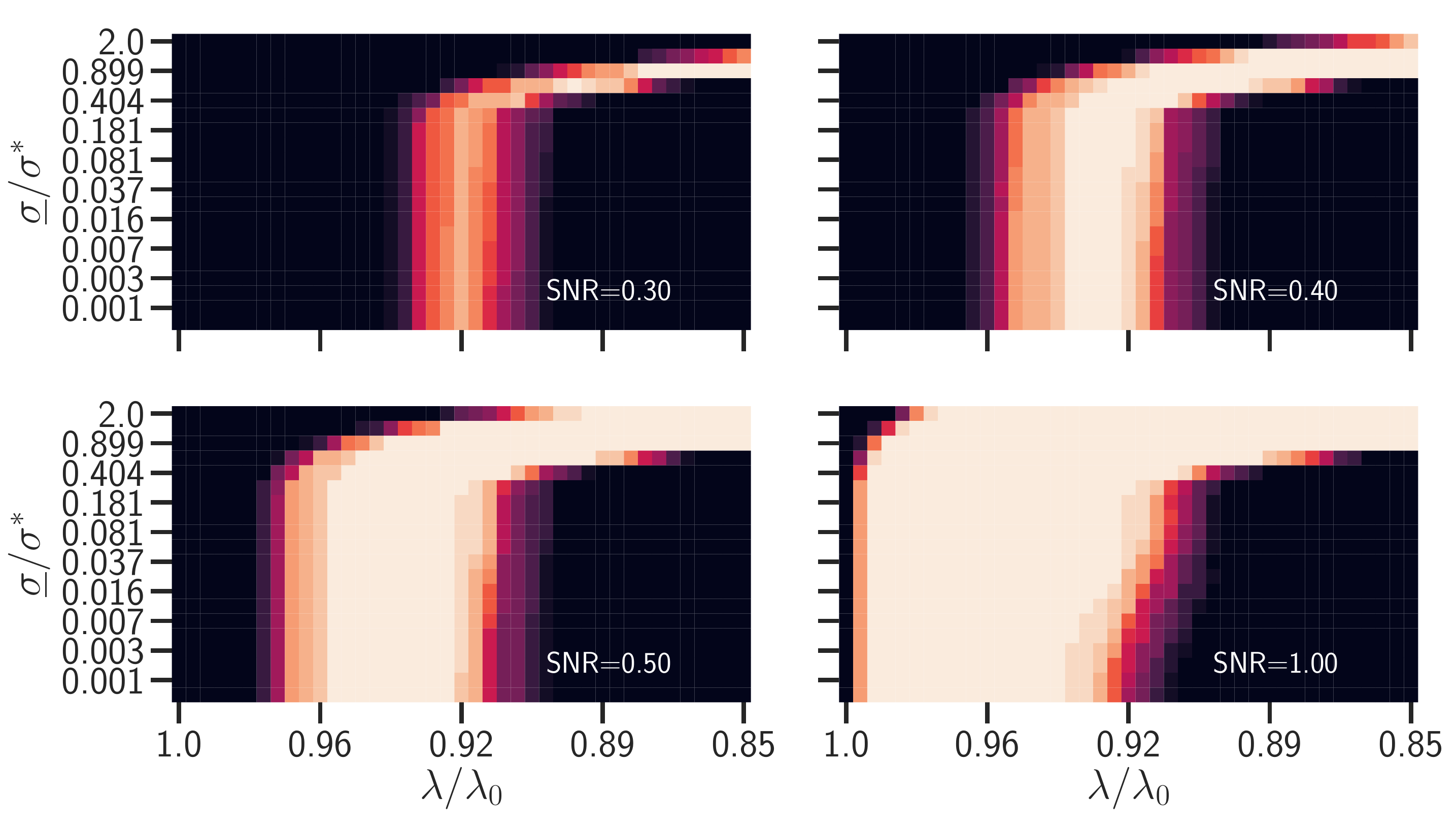}
    \end{minipage}
    \begin{minipage}{0.1\linewidth}
        \includegraphics[width=\figsize\linewidth]{colorbar_vert_hard_recov_SCL_rho_X_0_5}
    \end{minipage}
    \caption{(Synthetic data, $n=150$, $p=500$, $q=100$) Hard recovery loss for different
    values of SNR for the SGCL.
    }
    \label{fig:hard_recov_SGCL}
\end{figure}

\paragraph{Conclusion}
We have proved sup norm convergence rates and support recovery for a family of sparse estimators derived from the square-root Lasso.
We showed that they are pivotal too: the optimal regularization parameter does not depend on the noise level.
We showed that their smoothed versions retain these properties while being simpler to solve, and requiring more realistic assumptions to be analyzed.
These findings were corroborated numerically, in particular for the influence of the smoothing parameter.

\paragraph{Acknowledgments}
This work was funded by ERC Starting Grant SLAB ERC-StG-676943.
We would like to thank Karim Lounici for numerous discussions, suggestions and pointers.

\clearpage
\bibliographystyle{plainnat}
\bibliography{references_all}

\appendix

\onecolumn
\section{Technical lemmas}
\label{app_sec:technical_lem}

\begin{lemma}
    Let $\Psi$, $\alpha$ and $s$ satisfy \Cref{assum:mut_inco}, let $\hat \Beta$ be an estimator satisfying:
    $\norm{\Delta_{\cS_*^c}}_{2, 1}
         \leq  3 \norm{\Delta_{\cS_*}}_{2, 1}$,
    then:
  \begin{lemmaenum}[topsep=4pt,itemsep=4pt,partopsep=4pt,parsep=4pt]
    \item $\norm{\Delta_{\cS_*}}_F
          \leq
          \frac{\alpha }{ \alpha - 1} 4 \sqrt{s} \norm{\Psi \Delta}_{2, \infty}
          \enspace ,$ \label{lem:bnd_l2_l_infty}
    \item  $\norm{\Delta}_{2, 1}
            \leq
            \frac{\alpha}{ \alpha - 1} 16 s \norm{\Psi \Delta}_{2, \infty}
            \enspace ,$  \label{lem:bnd_l_1_l_infty}
    \item $
        \normin{\Delta}_{2, \infty}
        \leq
        \left(1 + \frac{16}{7(\alpha - 1) } \right) \norm{\Psi \Delta}_{2, \infty} \enspace.$ \label{lem:bnd_l_infty_l_infty}
  \end{lemmaenum}
\end{lemma}

\begin{proof}
  For {\Cref{lem:bnd_l2_l_infty}}, the idea is to upper and lower bound $\tfrac{1}{n} \norm{X \Delta}_F^2$.  
  First we bound $\norm{\Delta}_{2, 1}$:
  \begin{align}
      \norm{\Delta}_{2, 1} &= \norm{\Delta_{S_*^c}}_{2, 1} + \norm{\Delta_{S_*}}_{2, 1} \nonumber \\
      &\leq 4 \norm{\Delta_{S_*}}_{2, 1}\nonumber \\
      &\leq 4 \sqrt{s} \norm{\Delta_{S_*}}_F \enspace. \label{eq:link_l_1_l2}
  \end{align}
  Now we can upper bound $\tfrac{1}{n} \norm{X \Delta}_F^2$ with H\"older inequality and \Cref{eq:link_l_1_l2}:
  \begin{align}
      \tfrac{1}{n} \norm{X \Delta}_F^2
      & = \langle \Delta, \Psi \Delta \rangle \nonumber \\
      & \leq \norm{\Delta}_{2, 1} \norm{\Psi \Delta}_{2, \infty} \nonumber \\
      & \leq 4\sqrt{s}\norm{\Delta_{S_*}}_F \norm{\Psi \Delta}_{2, \infty} \enspace. \label{eq:link_fro_2_inf}
  \end{align}
  By \Cref{eq:sparse_conditionning} and \Cref{eq:link_fro_2_inf}:
  \begin{align}
      (1-\tfrac{1}{\alpha})\norm{\Delta_{S_*}}_F^2
      &\leq \tfrac{1}{n} \norm{X \Delta}_F^2
      \nonumber \\
      & \leq 4\sqrt{s}\norm{\Delta_{S_*}}_F \norm{\Psi \Delta}_{2, \infty}
      \nonumber\\
      \norm{\Delta_{S_*}}_F
      &\leq
      \frac{\alpha }{ \alpha - 1} 4 \sqrt{s} \norm{\Psi \Delta}_{2, \infty}
      \enspace .
  \end{align}
  \Cref{lem:bnd_l_1_l_infty} is a direct consequence of \Cref{eq:link_l_1_l2,lem:bnd_l2_l_infty}:
  \begin{align}
      \norm{\Delta}_{2, 1}
      &\leq 4 \sqrt{s} \norm{\Delta_{S_*}}_F \nonumber \\
      &\leq
      \frac{\alpha }{ \alpha - 1} 16 s \norm{\Psi \Delta}_{2, \infty}
      \enspace .
  \end{align}
  Finally, for \Cref{lem:bnd_l_infty_l_infty}, for any $j \in [p]$,
  \begin{align}
      \left(\Psi \Delta \right)_{j:}
      &=
      \Delta_{j:}
      + \textstyle{\sum_{j' \neq j}} \Psi_{j'j} \Delta_{j':}   \nonumber\\
      || \left(\Psi \Delta \right)_{j:} - \Delta_{j:} ||_2
      & \leq
      \textstyle{\sum_{j' \neq j}} |\Psi_{jj'}| \times ||\Delta_{j':}||_2
     \nonumber\\
      ||\left(\Psi \Delta \right)_{j:}
      - \Delta_{j:} ||_2
      &\leq
       \frac{1}{7 \alpha s} \textstyle{\sum_{j' \neq j}}
       \normin{\Delta_{j':}}  \nonumber\\
       \normin{\Delta}_{2, \infty}
      \leq &
       \normin{\Psi \Delta }_{2, \infty}
       + \tfrac{1}{7 \alpha s} \normin{\Delta}_{2, 1}   \nonumber \\
      \leq &
       \left(1 + \tfrac{16}{7(\alpha - 1) } \right) \norm{\Psi \Delta}_{2, \infty} \enspace .
   \end{align}
   using \Cref{assum:mut_inco,lem:bnd_l_1_l_infty}.
\end{proof}

\begin{lemma}\label{lem:link_hat_Epsilon_smooth_sqrt}
  Let \Cref{assum:high_noise} be true, on $\cA_1$ we have for \Cref{pb:multitask_smoothed_sqrt}:
  \begin{equation}
    \tfrac{1}{\sqrt{nq}} \normin{\hat \Epsilon}_F
    \leq
    ( 2 + \eta) \sigma^* \enspace .
  \end{equation}
\begin{proof}
For the \Cref{pb:multitask_smoothed_sqrt}, by the minimality of the estimator we have on $\cA_1$:
\begin{align*}
  \left (
    \norm{\cdot}_F \infconv \big(\tfrac{1}{2 \sigmamin} + \tfrac{\sigmamin}{2} \big) \normin{\cdot}_F^2
    \right )
    (\tfrac{\hat \Epsilon}{\sqrt{nq}} )
  + \lambda \normin{\hat \Beta}_{2, 1}
  &\leq
  \left (
    \norm{\cdot}_F \infconv \big(\tfrac{1}{2 \sigmamin} \normin{\cdot}_F^2 + \tfrac{\sigmamin}{2}) \right ) (\tfrac{ \Epsilon}{\sqrt{nq}})
  + \lambda \normin{\Beta^*}_{2, 1}
\end{align*}
\begin{align*}
  \tfrac{1}{\sqrt{nq}} \normin{\hat \Epsilon}_F
  + \lambda \normin{\hat \Beta}_{2, 1}
  &\leq
  \left (
    \normin{\cdot}_F \infconv \big(\tfrac{1}{2 \sigmamin} \normin{\cdot}_F^2 + \tfrac{\sigmamin}{2} ) (\tfrac{ \Epsilon}{\sqrt{nq}})
  \right )
  + \lambda \normin{\Beta^*}_{2, 1}
  && \text{since }  \normin{\cdot}_F \leq \left (
    \normin{\cdot}_F \infconv \big(\tfrac{1}{2 \sigmamin} \normin{\cdot}_F^2
    + \tfrac{\sigmamin}{2} \big)
    \right )  \\
  \tfrac{1}{\sqrt{nq}} \normin{\hat \Epsilon}_F
  &\leq
  \tfrac{1}{\sqrt{nq}} \normin{\Epsilon}_F
  + \lambda \normin{\Beta^*}_{2, 1}
  && \text{since } \tfrac{1}{\sqrt{nq}} \normin{\Epsilon}_F \geq \frac{\sigma^*}{\sqrt{2}} \geq \sigmamin\\
  &\leq  2 \sigma^* + \lambda \normin{\Beta^*}_{2, 1}
  && \text{since } \tfrac{1}{\sqrt{nq}} \normin{\Epsilon}_F \leq 2 \sigma^*\\
  & \leq 2\sigma^* + (1 + \eta) \sigma^*
  && \text{since } \lambda \normin{\Beta^*}_{2, 1} \leq (1 + \eta) \sigma^* \\
  \tfrac{1}{\sqrt{nq}} \normin{\hat \Epsilon}_F
  & \leq (3 + \eta ) \sigma^*  \enspace. \label{eq:link_f_sigma}
\end{align*}
\end{proof}
\end{lemma}
\begin{lemma}\label{lem:bound:residuals_sgcl}
  For \Cref{pb:conco_multivar,pb:smoothed_multivar_sqrt} we have:
    \begin{equation}
      \normin{X^\top \hat \Epsilon}_{2, \infty}
      \leq \normin{\hat S }_2 \normin{ X^\top \hat S^{-1}\hat \Epsilon}_{2, \infty}\label{eq:bound:XEpsilon_smooth_conco_multivar}\enspace.
  \end{equation}
\end{lemma}

\begin{proof}

  \textbf{Concomitant multivariate square-root} ($\hat S = (\hat \Epsilon \hat \Epsilon^\top)^{1 / 2} $)

  We recall that  $UDV^\top$ is a singular value decomposition of $\frac{1}{\sqrt{q}} \hat \Epsilon$, with
  $D = \diag(\gamma_i) \in \bbR^{r\times r}$,
  $U \in \bbR^{n \times r}$
  and $V \in \bbR^{q \times r}$ such that $U^\top U = V^\top V = \Id_r$.

  We have, observing that $\hat S^{-1} \hat \Epsilon = (U D^2 U^\top)^{-1/2} \sqrt{q} U D V^\top = \sqrt{q} U V^\top$:
  \begin{align*}
      X^\top \hat \Epsilon
      &= \sqrt{q} X^\top U D V^\top \\
      &= \sqrt{q} X^\top U V^\top V D V^\top \\
      &= X^\top \hat S^{-1}\hat \Epsilon V D V^\top \numberthis \enspace.
  \end{align*}
  Therefore,
  \begin{align*}
      \normin{X^\top \hat \Epsilon}_{2, \infty}
      &\leq \normin{V D V^\top}_2 \normin{ X^\top \hat S^{-1}\hat \Epsilon}_{2, \infty} \\
      &\leq \normin{\hat S }_2 \normin{ X^\top \hat S^{-1}\hat \Epsilon}_{2, \infty} \numberthis \label{eq:bound:XEpsilon_smoothed_conco_multivar}\enspace.
  \end{align*}

  \Cref{eq:bound:XEpsilon_conco_multivar} also holds for \Cref{pb:smoothed_multivar_sqrt}:

\textbf{Smoothed concomitant multivariate square-root} (
  $\hat S = U \diag([\gamma_i]_{\sigmamin}^{\sigmamax}) U^\top  $)

We recall that  $UDV^\top$ is a singular value decomposition of $\frac{1}{\sqrt{q}} \hat \Epsilon$, with
$D = \diag(\gamma_i) \in \bbR^{n\times n}$,
$U \in \bbR^{n \times n}$
and $V \in \bbR^{q \times n}$ such that $U^\top U = V^\top V = \Id_n$.

We have, observing that $\hat S^{-1} \hat \Epsilon = U \diag([\gamma_i]_{\sigmamin}^{\sigmamax})^{-1} U^\top \sqrt{q} U D V^\top = \sqrt{q} U \diag([\gamma_i]_{\sigmamin}^{\sigmamax})^{-1} \diag(\gamma_i) V^\top$:
\begin{align*}
    X^\top \hat \Epsilon
    &= \sqrt{q} X^\top U \diag(\gamma_i)  V^\top \\
    &= \sqrt{q} X^\top U \diag([\gamma_i]_{\sigmamin}^{\sigmamax})^{-1} \diag(\gamma_i)  \diag([\gamma_i]_{\sigmamin}^{\sigmamax}) V^\top \\
    &= \sqrt{q} X^\top U \diag([\gamma_i]_{\sigmamin}^{\sigmamax})^{-1} \diag(\gamma_i)  V^T V\diag([\gamma_i]_{\sigmamin}^{\sigmamax}) V^\top \\
    &= \sqrt{q} X^\top \hat S^{-1} \hat \Epsilon V\diag([\gamma_i]_{\sigmamin}^{\sigmamax}) V^\top
\end{align*}
Therefore,
\begin{align*}
    \normin{X^\top \hat \Epsilon}_{2, \infty}
    &\leq \normin{ X^\top \hat S^{-1}\hat \Epsilon}_{2, \infty} \normin{V \diag([\gamma_i]_{\sigmamin}^{\sigmamax}) V^\top}_2  \\
    &\leq \normin{ X^\top \hat S^{-1}\hat \Epsilon}_{2, \infty} \normin{\hat S }_2 \numberthis \enspace.
\end{align*}

\end{proof}


\onecolumn
\section{Concentration inequalities}
\label{app_sec:concentration}

%
The following theorem is a powerful tool to show a lot of concentration inequalities:
\begin{theorem}[{\cite[Thm B.6 p. 221]{Giraud14}}] \label{thm:esp_lipschitz_funct}
  Assume that $F: \bbR^d \to \bbR$ is 1-Lipschitz and $z$ has $\cN(0, \sigma^2 \Id_d)$ as a distribution, then there exists a variable $\xi$, exponentially distributed with parameter 1, such that:
  \begin{equation}
    F(z) \leq \bbE[F(z)] + \sqrt{2\xi} \enspace.
  \end{equation}
\end{theorem}
\begin{lemma}\label{lem:concentration}
    \begin{lemmaenum}
        \item \label{lem:control_event_lasso}
        Let $\cC_1 \eqdef \left \{ \frac{1}{n} \normin{X^\top \varepsilon}_\infty \leq \frac{\lambda}{2} \right \}$.
        Take $\lambda = A \sigma^* \sqrt{(\log p) / n}$ and $A > 2 \sqrt{2}$, then:
        \begin{equation}
          \bbP (\cC_1) \geq 1 - 2 p^{1-\frac{A^2}{8}} \enspace .
        \end{equation}
        \item \label{lem:control_event_sqrt_lasso}
        Let $\cC_1' \eqdef
        \left \{ \frac{1}{n} \normin{X^\top \varepsilon}_\infty \leq \frac{\lambda}{2} \frac{\sigma^*}{\sqrt{2}} \right \}$.
        Take $\lambda = A  \sqrt{(2\log p) / n}$ and $A > 2 \sqrt{2}$, then:
        \begin{equation}
          \bbP (\cC_1) \geq 1 - 2 p^{1-\frac{A^2}{8}} \enspace .
        \end{equation}
        \item \label{lem:control_event_mtl_lasso}
        Let $\cC_2 \eqdef \left \{ \frac{1}{nq} \normin{X^\top \Epsilon}_{2,\infty} \leq \frac{\lambda}{2} \right \}$.
        Take $\lambda = \frac{2\sigma^*}{\sqrt{nq}} \left (1 + A \sqrt{\frac{\log p}{q}} \right )$ and $A > \sqrt{2}$, then:
        \begin{equation}
          \bbP (\cC_2) \geq 1 - p^{1 - A^2/2} \enspace .
        \end{equation}
        Another possible control is \citet[Proof of Lemma 3.1, p. 6]{Lounici_Pontil_Tsybakov_vandeGeer09}.
        Let $A > 8$ and $\lambda = \frac{2 \sigma^*}{\sqrt{n q}} \sqrt{1 + \frac{A \log p}{\sqrt{q}} }$, then:
        \begin{equation}
          \bbP(\cC_2) \geq 1 - p^{\min(8 \log p, A \sqrt{q} / 8)} \enspace .
        \end{equation}
        \item \label{lem:control_event_mtl_sqrt_lasso}
        Let $\cC_3 \eqdef \left \{ \frac{1}{nq} \normin{X^\top \Epsilon}_{2,\infty} \leq \frac{\lambda \sigma^*}{2\sqrt{2}} \right \}$.
        Take $\lambda = \frac{2\sqrt{2}}{ \sqrt{nq}} \left (1 + A \sqrt{\frac{\log p}{q}} \right )$ and $A > \sqrt{2}$, then:
        \begin{equation}
          \bbP (\cC_3) \geq 1 - p^{1- A^2/2} \enspace .
        \end{equation}
        \item \label{lem:control_event_chi_square}
        Let $\cC_4 \eqdef \left \{ \frac{\sigma^*}{\sqrt{2}} < \frac{\normin{\Epsilon}_F}{\sqrt{nq}} < 2 \sigma^* \right\}$.
        Then:
        \begin{equation}
          \bbP (\cC_4) \geq 1 - (1 + e^2) e^{-nq/24} \enspace .
        \end{equation}
        \item \label{lem:control_min_sum_gaussian}
        Let $\cC_5 \eqdef \left \{ \left (\dfrac{\Epsilon \Epsilon^\top}{q} \right )^{\frac{1}{2}} \succ \dfrac{\sigma^*}{\sqrt{2}} \right\}$. Then with $c\geq \frac{1}{32}$:
        \begin{equation}
          \bbP(\cC_5) \geq 1 - n e^{-cq/(2n)} \enspace .
        \end{equation}
        \item \label{lem:control_max_sum_gaussian}
        Let $\cC_6 \eqdef \left \{
          2 \sigma^* \succ
          \left (
        \dfrac{\Epsilon \Epsilon^\top}{q} \right )^{\frac{1}{2}}
          \right\}$. Then with $c\geq \frac{1}{32}$:
        \begin{equation}
          \bbP(\cC_6) \geq 1 - n e^{-cq / n} \enspace .
        \end{equation}
        \item \label{lem:control_A1}
        Let us recall that $\cA_1 = \left \{ \tfrac{\normin{X^\top \Epsilon}_{2, \infty}}{\sqrt{nq}\normin{\Epsilon}_F} \leq \tfrac{\lambda}{2}   \right \}
          \cap \left \{ \tfrac{\sigma^*}{\sqrt{2}} < \tfrac{\normin{\Epsilon}_F}{\sqrt{nq}} < 2 \sigma^* \right \}$, we have
          \begin{equation}
            \bbP(\cA_1) \geq 1 - p^{1- A^2/2} - (1 + e^2) e^{-nq/24}
            \enspace .
          \end{equation}
        \item \label{lem:control_A2}
          Let us recall that $ \cA_2 = \left \{ \tfrac{\normin{X^\top \Epsilon  }_{2, \infty}}{nq} \leq \tfrac{\lambda \sigma^*}{2 \sqrt{2}} \right \}
        \cap \left \{ 2 \sigma^* \Id_q
          \succ (\tfrac{\Epsilon \Epsilon^\top}{n})^{\frac{1}{2}}
          \succ \tfrac{\sigma^*}{\sqrt{2}} \Id_q \right \}$
    \end{lemmaenum}
      \end{lemma}
  \begin{proof}
     \Cref{lem:control_event_lasso}:
    \begin{align}
        \bbP\left(\cC_1^c\right)
        & \leq p \, \bbP \left(\left|X_{:1}^\top \varepsilon\right| \geq n\lambda / 2\right) \nonumber \\
        & \leq p \, \bbP \left(\left|\varepsilon_1 \right| \geq \sqrt{n} \lambda / 2\right) \nonumber & \text{ $(\norm{X_{:1}}=\sqrt{n})$ }\\
        & \leq p \, \bbP \left(\left|\varepsilon_1 \right| / \sigma \geq \sqrt{n} \lambda / (2\sigma)\right) \nonumber \\
        & \leq 2 p \exp \left(-\frac{n \lambda^{2}}{8 \sigma^{2}}\right) & \text{(\Cref{thm:esp_lipschitz_funct})}\nonumber \\
        & \leq 2 p^{1-\frac{A^{2}}{8}} & \text{($\lambda = A \sigma \sqrt{(\log p) / n}$)} \enspace.
    \end{align}
    \Cref{lem:control_event_sqrt_lasso} is a direct consequence of \Cref{lem:control_event_lasso}.

    \Cref{lem:control_event_mtl_lasso}:
    since $\Epsilon$ is isotropic, the law of $u^\top \Epsilon$ is the same for all vectors $u \in \bbR^n$ of same norm.
    In particular, $X_{:1}^\top \Epsilon$ and $\sqrt{n} e_1^\top \Epsilon = \sqrt{n} \Epsilon_{1:}$ have the same law.

    The variable $\frac{1}{\sigma} \normin{\Epsilon_{1 :}}_2$ is a \emph{chi variable with $q$ degrees of freedom}, and
    \begin{equation}
      \frac{1}{\sigma} \bbE[\normin{\Epsilon_{1 :}}_2] = \frac{\sqrt{2}\Gamma(\tfrac{q + 1}{2})}{\Gamma(\frac{q}{2})} \in \left[ \frac{q}{\sqrt{q + 1}}, \sqrt{q} \right] \enspace,
    \end{equation}
    where the bound can be proved by recursion.
    We have:
    \begin{align}
        \bbP\left(\cC_2^c\right)
        & \leq p \, \bbP \left(\normin{X_{:1}^\top \Epsilon}_2 \geq q n\lambda / 2\right) \nonumber \\
        & \leq p \, \bbP \left( \normin{\Epsilon_{1:}}_2 \geq q \sqrt{n} \lambda / 2\right) \nonumber & \text{(by isotropy of $\Epsilon$)} \\
        & \leq p \, \bbP \left( \normin{\Epsilon_{1:}}_2
        \geq \sigma \sqrt{q} + A \sigma  \sqrt{\log p} \right) \nonumber & \text{($\lambda = \tfrac{2\sigma}{q \sqrt{n}}(\sqrt{q} + A \sqrt{\log p})$)}\\
        & \leq p \, \bbP \left( \normin{\Epsilon_{1:}}_2
        \geq \bbE(\normin{\Epsilon_{1:}}_2) + A \sigma \sqrt{\log p} \right) \nonumber & \text{($ \sigma \sqrt{q} \geq \bbE(\normin{\Epsilon_{1:}}_2)$}\\
        & \leq p^{1 - \tfrac{A^2}{2}} &  \text{(\Cref{thm:esp_lipschitz_funct}) \enspace.}
    \end{align}
    The proof of the other control of $\cA_2$ can be found in \citet[Proof of Lemma 3.1, p. 6]{Lounici_Pontil_Tsybakov_vandeGeer09}.

    \Cref{lem:control_event_mtl_sqrt_lasso} is a direct consequence of \Cref{lem:control_event_mtl_lasso}.

    Proof of \Cref{lem:control_event_chi_square} can be found in \citet[Proof of Lemma 5.4 p. 112]{Giraud14}, who control the finer event $\{ \frac{\sigma}{\sqrt{2}} \leq \frac{\normin{\varepsilon}}{\sqrt{n}} \leq (2 - \frac{1}{\sqrt{2}}) \sigma \}$.

    Proof of \Cref{lem:control_min_sum_gaussian,lem:control_max_sum_gaussian} are particular cases of \citet[Cor. 7.2, p. 15]{Gittens_Tropp11}.

    Proof of \Cref{lem:control_A1} is done using \Cref{lem:control_event_mtl_sqrt_lasso,lem:control_event_chi_square}.
    Indeed we have $A_1 \supset
    \left \{ \tfrac{1}{nq} \normin{X^\top \Epsilon}_{2, \infty} \leq \tfrac{\lambda \sigma^*}{2 \sqrt{2}}  \right \}
   \cap \left \{ \frac{\sigma^*}{\sqrt{2}} < \tfrac{\normin{\Epsilon}_F}{\sqrt{nq}} < 2 \sigma^* \right \} = \cC_3 \cap \cC_4$.
   Hence $\bbP (\cC_1) \geq 1 - \bbP(\cC_3^c) - \bbP(\cC_4^c) \geq 1 - p^{1- A^2/2} - (1 + e^2) e^{-nq/24}$.

   Proof of \Cref{lem:control_A2} is done using \Cref{lem:control_event_mtl_sqrt_lasso,lem:control_min_sum_gaussian,lem:control_max_sum_gaussian}.
   Indeed $A_2
  = \cC_3 \cap \cC_5 \cap \cC_6$.
  Hence $\bbP (\cA_2) \geq 1 - \bbP(\cC_3^c) - \bbP(\cC_5^c) - \bbP(\cC_6^c) \geq 1 - p^{1- A^2/2} - 2 n e^{-cq/(2n)}$.
  \end{proof}
%


\section{Single task cases}
\label{app_sec:single_case}
%
The results in \Cref{prop:sqrtmtl_bds_est} are proposed in a multitasks settings, thus they still hold in the single-task setting ($q=1$).
However it is possible to achieve tighter convergence rates in the single task setting, \ie when $q=1$.
%
%
\subsection{Square root Lasso}
\label{app_sub:sqrt_lasso_single_task}
%
%
\begin{proposition}\label{prop:sing_sqrt_lasso_bds_est}
    Let \Cref{assum:gauss_noise} be satisfied, let $\alpha$ satisfy \Cref{assum:mut_inco} and let $\eta$ satisfy \Cref{assum:high_noise}.
    Let $C = 2 \left ( 1 + \tfrac{16}{7(\alpha - 1) } \right )$ and
    \begin{equation}
      \lambda = A  \sqrt{2\log p / n} \enspace .
    \end{equation}
    Then with probability at least $1 - p^{1 - A^2/8} - (1 + e^2) e^{-n/24}$,
    \begin{equation}
        \tfrac{1}{q} \normin{\hat \beta - \beta^*}_{2, \infty}
            \leq  C (2 + \eta) \lambda \sigma
             \enspace .
    \end{equation}
    Moreover if
    \begin{equation}
      \min_{j \in \cS^*} |\beta_{j}^*|
      >
      2 C (2 + \eta) \lambda \sigma \enspace,
    \end{equation}
    then with the same probability:
    \begin{equation}
        \hat \cS = \{
          j \in [p] :
          |\hat \beta_j|
            > C (2 + \eta) \lambda \sigma
        \} \enspace
      \end{equation}
    estimate correctly the true sparsity pattern:
    \begin{equation}
      \hat \cS = \cS^* \enspace .
    \end{equation}
\end{proposition}
\begin{proof}
    All the inequalities leading to \Cref{eq:bound_1_q_psi_delta} still hold.
    The control of the event $\cA_2$ can be tighter in the single-task case.
    Since $\cA_2 \supset \cC_1' \cup \cC_4$, with $\lambda = A  \sqrt{2\log p / n}$ this leads to:
    \begin{equation}
    \cP(\cA_2) \geq 1 - p^{1 - A^2 / 8} - (1 + e^2) e^{-n /24} \enspace .
    \end{equation}
\end{proof}




\end{document}